\documentclass[11pt,a4paper]{article}
\usepackage[hyperref]{emnlp2020}
\usepackage{times}
\usepackage{latexsym}

\usepackage{microtype}

\usepackage{xspace}
\usepackage{enumitem}
\usepackage{smile}
\usepackage{url}

\usepackage{booktabs} 
\usepackage{diagbox}

\usepackage{multirow}
\usepackage{balance}
\usepackage{amsfonts}
\usepackage{amsmath}
\usepackage{amssymb}
\usepackage{graphicx}
\usepackage{subfigure}
\usepackage{microtype}
\usepackage{xspace}
\usepackage{wrapfig,lipsum}
\usepackage[noend,ruled]{algorithm2e}
\usepackage{makecell}
\usepackage[T1]{fontenc}
\usepackage{soul}
\usepackage{xcolor}
\setul{}{1.6pt}
\newcommand{\soutthick}[1]{%
      \st{#1}%
}
\usepackage{lipsum}

\usepackage{appendix}
\usepackage{tabularx, booktabs}
\newcolumntype{Y}{>{\centering\arraybackslash}X}
\newcolumntype{L}{>{\arraybackslash}X}

\def\aclpaperid{1552}
\aclfinalcopy  

\newcommand{\our}{\mbox{Admin}\xspace}
\newcommand{\eg}{\textit{e.g.}}
\newcommand{\ie}{\textit{i.e.}}

\newcommand{\smallsection}[1]{{\vspace{0.2cm}\noindent\textbf{#1.}}}

\newcommand{\LN}{{f_{\mbox{\scriptsize LN}}}}
\newcommand{\ATT}{{f_{\mbox{\scriptsize ATT}}}}
\newcommand{\SATT}{{f_{\mbox{\scriptsize S-ATT}}}}
\newcommand{\EATT}{{f_{\mbox{\scriptsize E-ATT}}}}
\newcommand{\FFN}{{f_{\mbox{\scriptsize FFN}}}}
\newcommand{\softmax}{{f_{s}}}
\newcommand{\ep}{^{(pe)}}
\newcommand{\pd}{^{(pd)}}
\newcommand{\eo}{^{(oe)}}
\newcommand{\od}{^{(od)}}
\newcommand{\xbteo}{{\xb^T}^{(oe)}}

\newcommand{\wf}{W^{(1)}}
\newcommand{\wff}{W^{(2)}}
\newcommand{\wk}{W^{(K)}}
\newcommand{\wv}{W^{(V_1)}}
\newcommand{\wq}{W^{(Q)}}
\newcommand{\wo}{W^{(V_2)}}

\newcommand{\hxb}{\mathbf{\hat{x}}}

\newcommand{\hab}{\mathbf{\hat{a}}}

\title{{U}nderstanding the {D}ifficulty of {T}raining {T}ransformers}

\author{
Liyuan Liu\textsuperscript{$\dagger\ddagger$}~~
Xiaodong Liu\textsuperscript{$\ddagger$}~~
Jianfeng Gao\textsuperscript{$\ddagger$}~~
Weizhu Chen\textsuperscript{$\mathsection$}~~
Jiawei Han\textsuperscript{$\dagger$}\\
\texttt{\small \{ll2, hanj\}@illinois.edu}
,~~\texttt{\small \{xiaodl,jfgao,wzchen\}@microsoft.com}
\\
\textsuperscript{$\dagger$}{University of Illinois at Urbana-Champaign} \\ \textsuperscript{$\ddagger$}Microsoft Research \\ \textsuperscript{$\mathsection$} Microsoft Dynamics 365 AI
}

\date{}

\begin{document}
\maketitle


\begin{abstract}
  

Transformers have proved effective in many NLP tasks.
However, their training requires non-trivial efforts regarding  designing cutting-edge optimizers and learning rate schedulers carefully (\eg, conventional SGD fails to train Transformers effectively). 
Our objective here is to understand \emph{what complicates Transformer training} from both empirical and theoretical perspectives. 
Our analysis reveals that unbalanced gradients are not the root cause of the instability of training.  
Instead, we identify an \emph{amplification effect} that influences training substantially--for each layer in a multi-layer Transformer model, 
heavy dependency on its residual branch makes training unstable, since it amplifies small parameter perturbations (\eg, parameter updates)
and results in significant disturbances in the model output.
Yet we observe that a light dependency limits the model potential and leads to inferior trained models. 
Inspired by our analysis, we propose \emph{\our} (\textbf{Ad}aptive \textbf{m}odel \textbf{in}itialization) to stabilize the early stage's training and unleash its full potential in the late stage. 
Extensive experiments show that \our is more stable, converges faster, and leads to better performance\footnote{Implementations are released at: \url{https://github.com/LiyuanLucasLiu/Transforemr-Clinic}}. 

\end{abstract}


\section{Introduction}

Transformers~\cite{Vaswani2017AttentionIA} have led to a series of breakthroughs in various deep learning tasks ~\cite{Devlin2019BERTPO,Velickovic2017GraphAN}. 
They do not contain recurrent connections and can parallelize all computations in the same layer, thus improving effectiveness, efficiency, and scalability.
Training Transformers, however, requires extra efforts.
For example, although stochastic gradient descent (SGD)
is the standard algorithm for conventional RNNs and CNNs, it converges to bad/suspicious local optima for Transformers~\cite{Zhang2019WhyAB}.
Moreover, comparing to other neural architectures, removing the warmup stage in Transformer training results in more severe consequences such as model divergence~\cite{Popel2018TrainingTF,Liu2019OnTV}.
Here, we conduct comprehensive analyses in empirical and theoretical manners to answer the question: \textit{what complicates Transformer training}.

\begin{figure}[t]
\centering
\includegraphics[width=\linewidth]{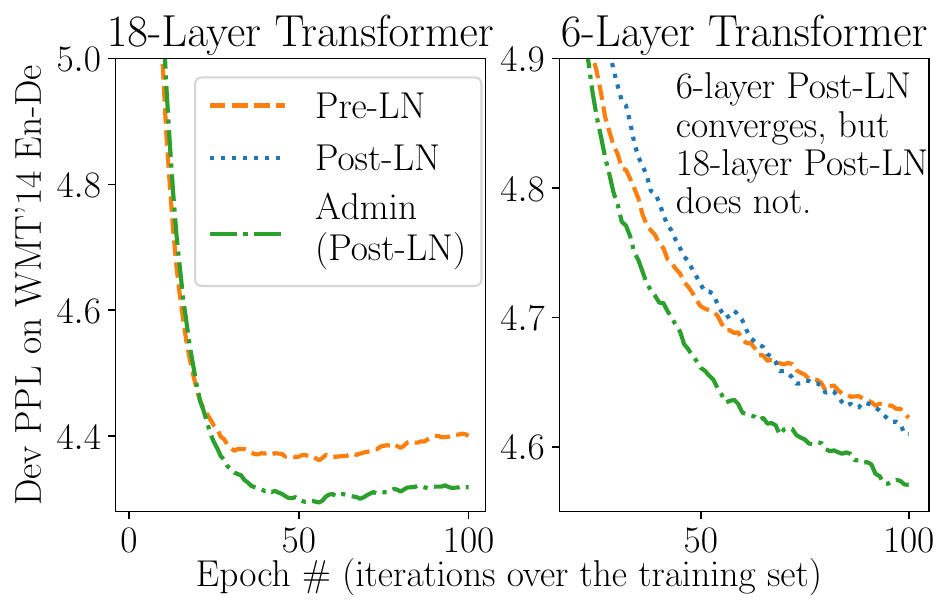}
\caption{Lacking enough robustness and stability, the 18-Layer Post-LN Transformer training (\ie the original architecture) diverges and is omitted in the left graph. \our not only stabilizes model training but unleashes the model potential for better performance. 
}
\label{fig:wmt14ende_devppl_6_18}
\end{figure}

Our analysis starts from the observation: the original Transformer (referred to as Post-LN) is less robust than its Pre-LN variant\footnote{As in Figure~\ref{fig:pre-post-diagram}, Post-LN places layer norm outside of residual blocks, and Pre-LN moves them to the inside.}~\cite{Baevski2018AdaptiveIR,Xiong2019OnLN,Nguyen2019TransformersWT}. 
We recognize that gradient vanishing issue is not the direct reason causing such difference, since fixing this issue alone cannot stabilize Post-LN training. 
It implies that, besides unbalanced gradients, there exist other factors influencing model training greatly. 

\begin{figure*}[t]
\centering
\includegraphics[width=\linewidth]{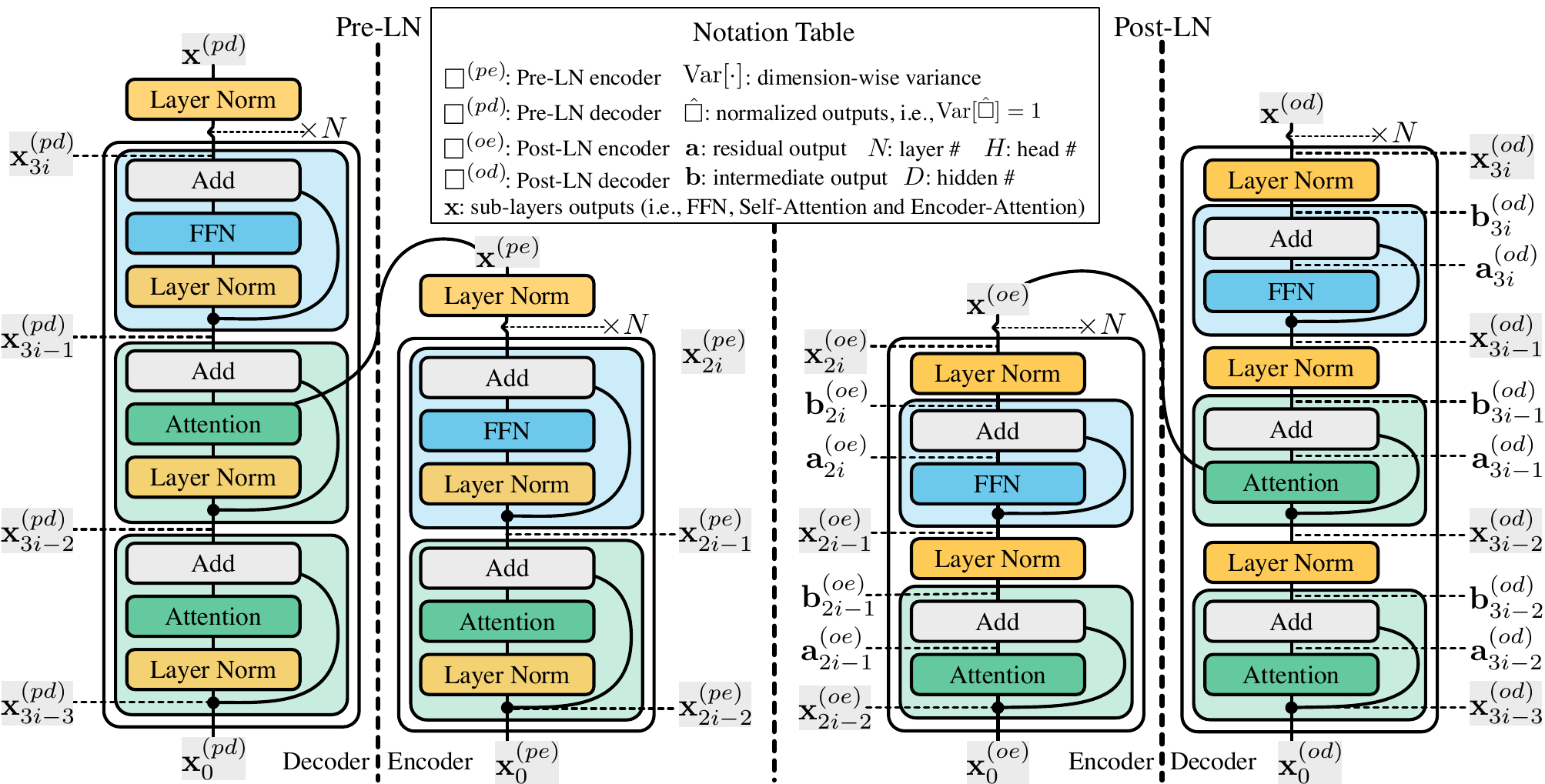}
\caption{The Architecture and notations of Pre-LN Transformers (Left) and Post-LN Transformers (Right).}
\label{fig:pre-post-diagram}
\end{figure*}
 
With further analysis, we recognize that for each Transformer residual block, the dependency on its residual branch\footnote{For a residual block $x+f(x)$, its shortcut output refers to $x$, its residual branch output refers to $f(x)$, and the dependency on its residual branch refers to $\frac{\Var[f(x)]}{\Var[x + f(x)]}$.} plays an essential role in training stability. 
First, we find that a Post-LN layer has a heavier dependency on its residual branch than a Pre-LN layer.
As in Figure~\ref{fig:layer-dependency}, at initialization, a Pre-LN layer has roughly the same dependency on its residual branch and any previous layer, whereas a Post-LN layer has a stronger dependency on its residual branch 
(more discussions are elaborated in Section~\ref{subsec:impact-ln}). 
We find that strong dependencies of Post-LN amplify fluctuations brought by parameter changes and destabilize the training (as in Theorem~\ref{theo:layer-dependency-and-shift} and Figure~\ref{fig:stability}). 
Besides, the loose reliance on residual branches in Pre-LN generally limits the algorithm's potential and often produces inferior models.  

In light of our analysis, we propose \our, an adaptive initialization method which retains the merits of Pre-LN stability without hurting the performance.
It restricts the layer dependency on its residual branches in the early stage and unleashes the model potential in the late stage.
We conduct experiments on IWSLT'14 De-En,  WMT'14 En-De, and WMT'14 En-Fr;
\our is more stable, converges faster, and achieves better performance. 
For example, without introducing any additional hyper-parameters, \our successfully stabilizes 72-layer Transformer training on WMT'14 En-Fr and achieves a 43.80 BLEU score.


\section{Preliminaries}

\begin{figure*}[t]
\centering
\includegraphics[width=\linewidth]{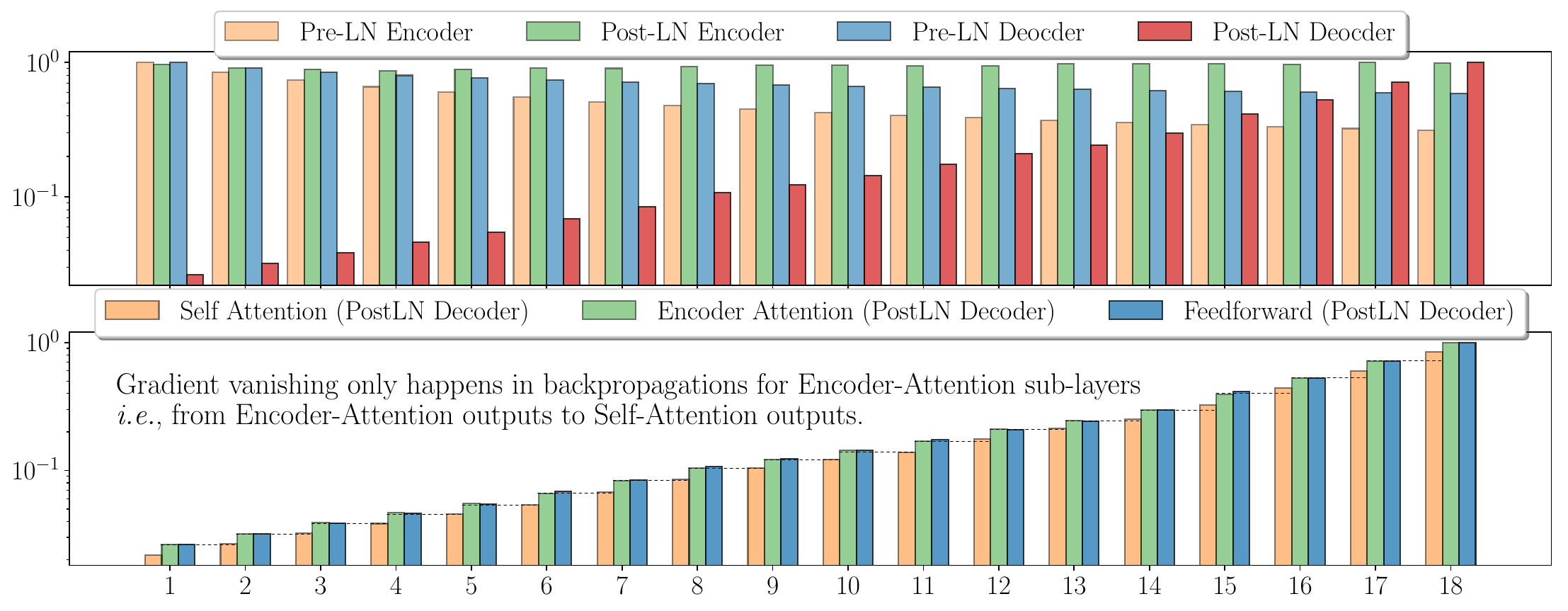}
\caption{Relative gradient norm histogram (on a log scale) of 18-layer Transformers on the WMT'14 En-De dataset, \ie, the gradient norm of sub-layer outputs, scaled by the largest gradient norm in the same network. 
}
\label{fig:gradient-histogram}
\end{figure*}

\smallsection{Transformer Architectures and Notations}
\label{subsec:architecture}
The Transformer architecture contains two types of sub-layers, \ie, Attention sub-layers and Feedforward (FFN) sub-layers. 
They are composed of mainly three basic modules~\cite{Vaswani2017AttentionIA}, \ie, Layer Norm ($\LN$), Multi-head Attention ($\ATT$), and Feedforward Network ($\FFN$).

As illustrated in Figure~\ref{fig:pre-post-diagram}, the Pre-LN Transformer and the Post-LN Transformer organize these modules differently. 
For example, a Pre-LN encoder organizes the Self-Attention sub-layer as $\xb\ep_{2i-1} = \xb\ep_{2i-2} + \SATT(\LN(\xb\ep_{2i-2}))$ and a Post-LN encoder as $\xb\eo_{2i-1} = \LN(\xb\eo_{2i-2} + \SATT(\xb\eo_{2i-2}))$, where $\xb^{(\cdot)}_{2i-2}$ is the input of the $i$-th Transformer layer and $\xb^{(\cdot)}_{2i-1}$ is the output of the $i$-th Self-Attention sub-layer.
Here, we refer  $\SATT(\LN(\xb\ep_{2i-2}))$ and $\SATT(\xb\eo_{2i-2})$ as the residual branches and their outputs as the residual outputs, in contrast to layer/sub-layer outputs, which integrates residual outputs and shortcut outputs. 

Notation elaborations are shown in Figure~\ref{fig:pre-post-diagram}.
In particular, we use superscripts to indicate network architectures (\ie, the Pre-LN Encoder), use subscripts to indicate layer indexes (top layers have larger indexes), all inputs and outputs are formulated as $\mbox{Sequence-Len} \times \mbox{Hidden-Dim}$.

\smallsection{Layer Norm}
Layer norm~\cite{Ba2016LayerN} plays a vital role in Transformer architecture. It is defined as $\LN(\xb) = \bgamma \frac{\xb - \mu }{\sigma} + \bnu$, where $\mu$ and $\sigma$ are the mean and standard deviation of 
$\xb$.

\smallsection{Feedforward Network}
Transformers use two-layer perceptrons
as feedforward networks,
\ie, $\FFN(\xb) = \phi(\xb \wf)\wff$, where $\phi(\cdot)$ is the non-linear function\footnote{Our analysis uses ReLU as the activation function, while \our can be applied to other non-linear functions.}, and $W^{(\cdot)}$ 
are parameters.

\smallsection{Multi-head Attention}
Multi-head Attentions allows the network to have multiple focuses in a single layer and plays a crucial role in many tasks~\cite{Chen2018TheBO}.
It is defined as (with $H$ heads):
$\ATT(\qb, \kb, \vb) = \sum_{h=1}^H \softmax(\qb \wq_h \wk_h \kb^T)\vb \wv_h \wo_h$,
where $\softmax$ is the row-wise softmax function and $W^{(\cdot)}_h$ are parameters. 
$\wq_h$ and $\wv_h$ are $D \times \frac{D}{H}$ matrices, $\wk_h$ and $\wo_h$ are $\frac{D}{H} \times D$ matrices, where $D$ is the hidden state dimension. 
Parameters without subscript refer the concatenation of all $H$-head parameters, \eg, $\wq = [\wq_1, \cdots, \wq_H]$. 
In Transformer,
this module is used in two different settings: Encoder-Attention ($\EATT(\xb) = \ATT(\xb, \xb^{(\cdot e)}, \xb^{(\cdot e)})$ and $\xb^{(\cdot e)}$ is the encoder output), and Self-Attention ($\SATT(\xb) = \ATT(\xb, \xb, \xb)$).


\section{Unbalanced Gradients}
\label{sec:gradient-distribution}

In this study, we strive to answer the question: \emph{what complicates Transformer training}.
Our analysis starts from the observation: Pre-LN training is more robust than Post-LN, while Post-LN is more likely to reach a better performance than Pre-LN. 
In a parameter grid search (as in Figure~\ref{fig:grid-search}), Pre-LN converges in all 15 settings, and Post-LN diverges in 7 out of 15 settings; when Post-LN converges, it outperforms Pre-LN in 7 out of 8 settings. 
We seek to reveal the underlying factor that destabilizes Post-LN training and restricts the performance of Pre-LN. 

\begin{figure*}[t]
\begin{minipage}{0.73\linewidth}
\includegraphics[width=\textwidth]{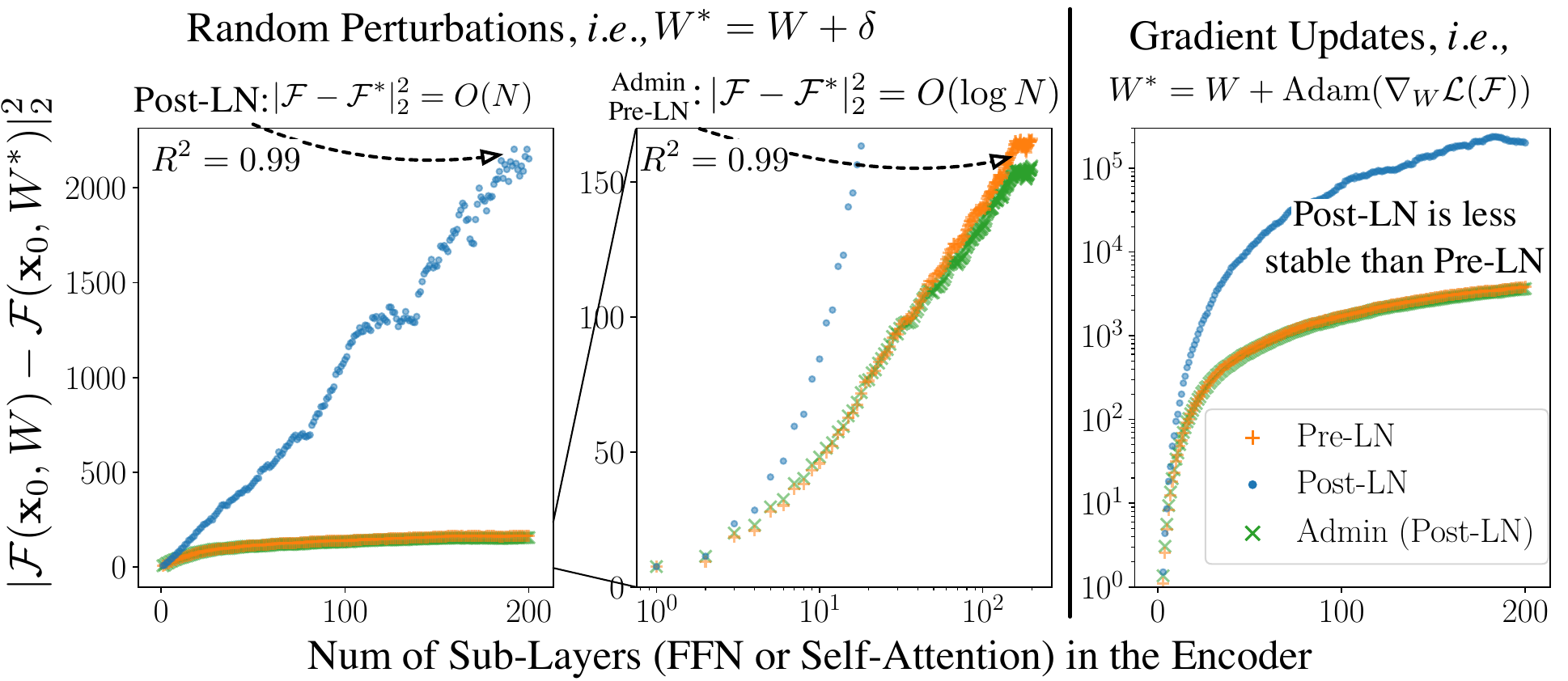}
\captionof{figure}{
Encoder output changes for parameter changes,
\ie, $|\cF(\xb_0, W) - \cF(\xb_0, W^*)|_2^2$ where $W^* - W$ is random perturbations (left) or gradient updates (right). 
Intuitively, very large $|\cF - \cF^*|$ indicates the training to be ill-conditioned.}
\label{fig:stability}
\end{minipage}
\,
\begin{minipage}{0.24\linewidth}
\includegraphics[width=\textwidth]{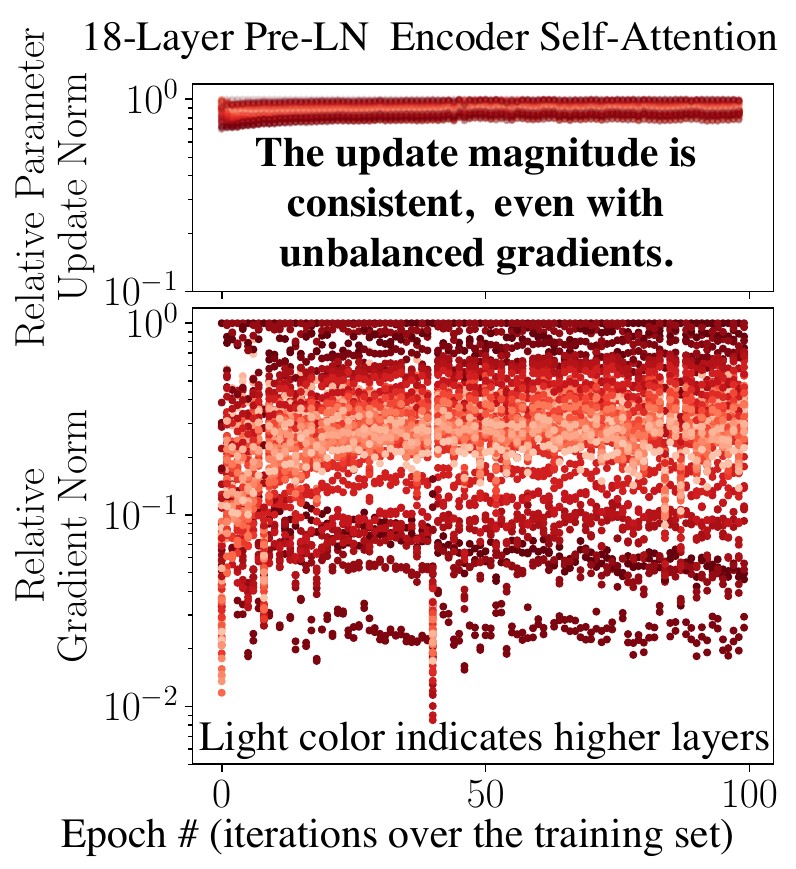}
 \captionof{figure}{Histogram of relative norm of gradient and $|W_{i+1} - W_{i}|$ where $W_i$ is the checkpoint saved after training for $i$ epochs.}
\label{fig:all_histogram}
\end{minipage}
\end{figure*}

In this section, we focus on the unbalanced gradients (\eg, gradient vanishing). 
We find that, although 
Post-LN suffers from gradient vanishing and Pre-LN does not, 
gradient vanishing is not the direct reason causing the instability of Post-LN. 
Specifically, we first theoretically and empirically establish that only Post-LN decoders suffer from gradient vanishing and Post-LN encoders do not.
We then observe that fixing the gradient vanishing issue alone cannot stabilize training.

\subsection{Gradients at Initialization}
\label{subsec:ini_grad}

As gradient vanishing can hamper convergence from the beginning, it has been regarded as the major issue causing unstable training. 
Also, recent studies show that this issue exists in the Post-LN Transformer, even after using residual connections~\cite{Xiong2019OnLN}. 
Below, we establish that only Post-LN decoders suffer from the gradient vanishing, and neither Post-LN encoders, Pre-LN encoders, nor Pre-LN decoders. 

\begin{table}[t]
\begin{center}
\small
\vspace{0.4cm}
\begin{tabularx}{\linewidth}{ *{4}{Y}}
\toprule
\textbf{Encoder} & \textbf{Decoder} & \textbf{Gradient} & \textbf{Training} \\ \midrule
Post-LN & Post-LN & Vanishing & Diverged \\
Post-LN & Pre-LN & \soutthick{Vanishing} & Diverged \\
Pre-LN & Pre-LN & \soutthick{Vanishing} & Converged \\
\bottomrule
\end{tabularx}
\end{center}
\caption{Changing decoders from Post-LN to Pre-LN fixes gradient vanishing, but does not stabilize model training successfully. Encoder/Decoder have 18 layers.}
\label{tbl:fixing-gradient-vanishing}
\end{table}

We use $\Delta \xb$ to denote gradients, \ie, $\Delta \xb = \frac{\partial \mathcal{L}}{\partial \xb}$ where $\mathcal{L}$ is the training objective.
Following previous studies~\cite{Glorot2010UnderstandingTD},
we analyze the gradient distribution at the very beginning of training and find only Encoder-Attention sub-layers in Post-LN suffers from gradient vanishing.  

First, we conduct analysis from a theoretical perspective. 
Similar to \citet{Xiong2019OnLN}, we establish that Pre-LN networks do not suffer from gradient vanishing (as elaborated in Appendix~\ref{subsec:preln-analysis}). 
Unlike \citet{Xiong2019OnLN}, we recognize that not all Post-LN networks suffer from gradient vanishing. 
As in Theorem~\ref{theo:post-ln-encoder-gradient}, we establish that Post-LN Encoder networks do not suffer from gradient vanishing. 
Detailed derivations are elaborated in Appendix~\ref{subsec:postln-encoder-analysis}. 

\begin{theorem}
For Post-LN Encoders, if $\bgamma$ and $\bnu$ in the Layer Norm are initialized as $1$ and $0$ respectively; all other parameters are initialized by symmetric distributions with zero mean; $\xb_{i}^{(oe)}$ and $\Delta \xb_i^{(oe)}$ are subject to symmetric distributions with zero mean; the variance of $\xb_{i}^{(oe)}$ is $1$ (\ie, normalized by Layer Norm); $\Delta \xb_i^{(oe)}$ and the derivatives of modules in $i$-th sub-layer are independent, we have $\Var[\Delta \xb_{i-1}] \geq \Var[\Delta \xb_{i}]$.
\label{theo:post-ln-encoder-gradient}
\end{theorem}

To make sure that the assumptions of Theorem~\ref{theo:layer-dependency-and-shift} match the real-world situation, we further conduct empirical verification. 
At initialization, we calculate $||\Delta \xb^{(\cdot)}_{i}||_2$ for 18-layer Transformers\footnote{Note if $\E[\Delta \xb^{(p\cdot)}_{i-1}] = 0$, $\Var[\Delta \xb^{(p\cdot)}_{i-1}] \approx |\Delta \xb^{(p\cdot)}_{i-1}|_2^2$.} and visualize
$\frac{||\Delta \xb^{(\cdot)}_{i}||_2}{\max_j ||\Delta \xb^{(\cdot)}_{j}||_2}$ 
in Figure~\ref{fig:gradient-histogram}.
It verifies that only Post-LN decoders suffer from the gradient vanishing. 
Besides, we can observe that the dropping of gradient norms mostly happens in the backpropagation from encoder-attention outputs (encoder-attention bars) to its inputs (self-attention bars, since the output of self-attention is the input of encoder-attention). 
This pattern is further explained in Appendix~\ref{subsec:postln_decoder_analysis}. 

\subsection{Impact of the Gradient Vanishing}

Now, we explore whether gradient vanishing is the direct cause of training instability.

First, we design a controlled experiment to show the relationship between gradient vanishing and training stability. 
We construct a hybrid Transformer by combining a Post-LN encoder and a Pre-LN decoder. 
As in Section~\ref{subsec:ini_grad}, only Post-LN decoders suffer from gradient vanishing, but not Post-LN encoders. 
Therefore, this hybrid Transformer does not suffer from gradient vanishing. 
As shown in Table~\ref{tbl:fixing-gradient-vanishing}, fixing gradient vanishing alone (\ie, changing Post-LN decoders to Pre-LN decoders) fails to stabilize model training. 
This observation provides evidence supporting that the gradient vanishing issue is not the direct cause of unstable Post-LN training. 

Moreover, we observe
that gradients of all attention modules are unbalanced, while adaptive optimizers mostly address this issue.
As in Figure~\ref{fig:all_histogram}, adaptive optimizers successfully assign different learning rates to different parameters and lead to consistent update magnitudes even with unbalanced gradients.
It explains why the standard SGD fails in training Transformers (\ie, lacking the ability to handle unbalanced gradients) and necessitates using adaptive optimizers. 
More discussions are included in Appendix~\ref{subsec:attention_gradients}.


\section{Instability from Amplification Effect}
\label{sec:layer-dependency}

We find that unbalanced gradients are not the root cause of the instability of Post-LN, which implies the existence of other factors influencing model training.
Now, we go beyond gradient vanishing and introduce the \emph{amplification effect}.
Specifically, we first examine the difference between Pre-LN and Post-LN, including their early-stage and late-stage training. 
Then, we show that Post-LN's training instability is attributed to layer dependency's amplification effect, which intensifies gradient updates and destabilizes training.

\begin{figure}[t]
\centering
\includegraphics[width=\linewidth]{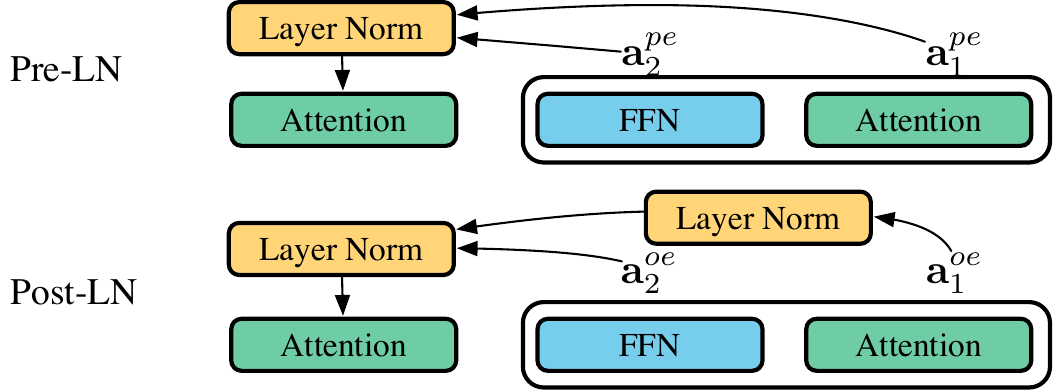}
\caption{
The major difference between Pre-LN and Post-LN is the position of layer norms. 
}
\label{fig:layer-norm-position}
\end{figure}

\subsection{Impact of Layer Norms Positions}
\label{subsec:impact-ln}

As described in Section~\ref{subsec:architecture}, both Pre-LN and Post-LN employ layer norm to regularize inputs and outputs. 
Different residual outputs are aggregated and normalized in residual networks before serving as inputs of other layers (\ie, residual outputs will be scaled to ensure the integrated input to have a consistent variance). 
To some extend, layer norm treats the variance of residual outputs as weights to average them.
For example, for Post-LN Self-Attention, we have
$\xb^{(o\cdot)}_{2i-1} = \frac{\xb^{(o\cdot)}_{2i-2} + \ab^{(o\cdot)}_{2i-1}}{\sqrt{\Var[\xb^{(o\cdot)}_{2i-2}] + \Var[\ab^{(o\cdot)}_{2i-1}]}}$
at initialization. 
Larger $\Var[\ab^{(o\cdot)}_{2i-2}]$ not only increases the proportion of $\ab^{(o\cdot)}_{2i-2}$ in $\xb^{(o\cdot)}_{2i-2}$ but decreases the proportion of other residual outputs.
Intuitively, this is similar to the weight mechanism of the weighted average.

The position of layer norms is the major difference between Pre-LN and Post-LN and makes them aggregate residual outputs differently (\ie, using different weights). 
As in Figure~\ref{fig:layer-norm-position}, all residual outputs in Pre-LN are only normalized once before feeding into other layers (thus only treating residual output variances as weights); in Post-LN, most residual outputs are normalized more than once, and different residual outputs are normalized for different times. 
For example, if all layers are initialized in the same way, output variances of different Pre-LN residual branches would be similar, and the aggregation would be similar to the simple average. 
Similarly, for Post-LN, nearby residual outputs are normalized by fewer times than others, thus having relatively larger weights. 
We proceed to calculate and analyze these weights to understand the impact of layer norm positions.

First, we use $\hab_i$ to refer $\frac{\ab_i}{\sqrt{\Var{\ab_i}}}$ (\ie, normalized outputs of $i$-th residual branch) and $\hxb_i$ to refer
$\frac{\xb_i}{\sqrt{\Var{\xb_i}}}$ (\ie, normalized outputs of $i$-th layer or normalized inputs of ($i$+1)-th residual branch).
Then, we describe their relationships as $\hxb_{i} = \sum_{j\leq i} \beta_{i, j} \hab_j$, where $\beta_{i, j}$ integrates scaling operations of all layer norms (including $\sqrt{\Var[\ab_i]}$). 
For example, Pre-LN sets $\beta_{i, j} = \frac{\sqrt{\Var[\ab_j]}}{\sqrt{\Var[\sum_{k\le i} \ab_k]}}$. 
Intuitively, $\beta_{i, j}$ describes the proportion of $j$-th residual branch outputs in $i$-th layer outputs, thus reflects the dependency among layers.

We visualize $\beta_{i, j}$ in Figure~\ref{fig:layer-dependency}. 
For a Post-LN layer, its outputs rely more on its residual branch from the initialization to the end.
At initialization, Pre-LN layer outputs have roughly the same reliance on all previous residual branches. 
As the training advances, each layer starts to rely more on its own residual outputs. 
However, comparing to Post-LN, Pre-LN layer outputs in the final model still has less reliance on their residual branches. 

\begin{figure}[t]
\centering
\includegraphics[width=\linewidth]{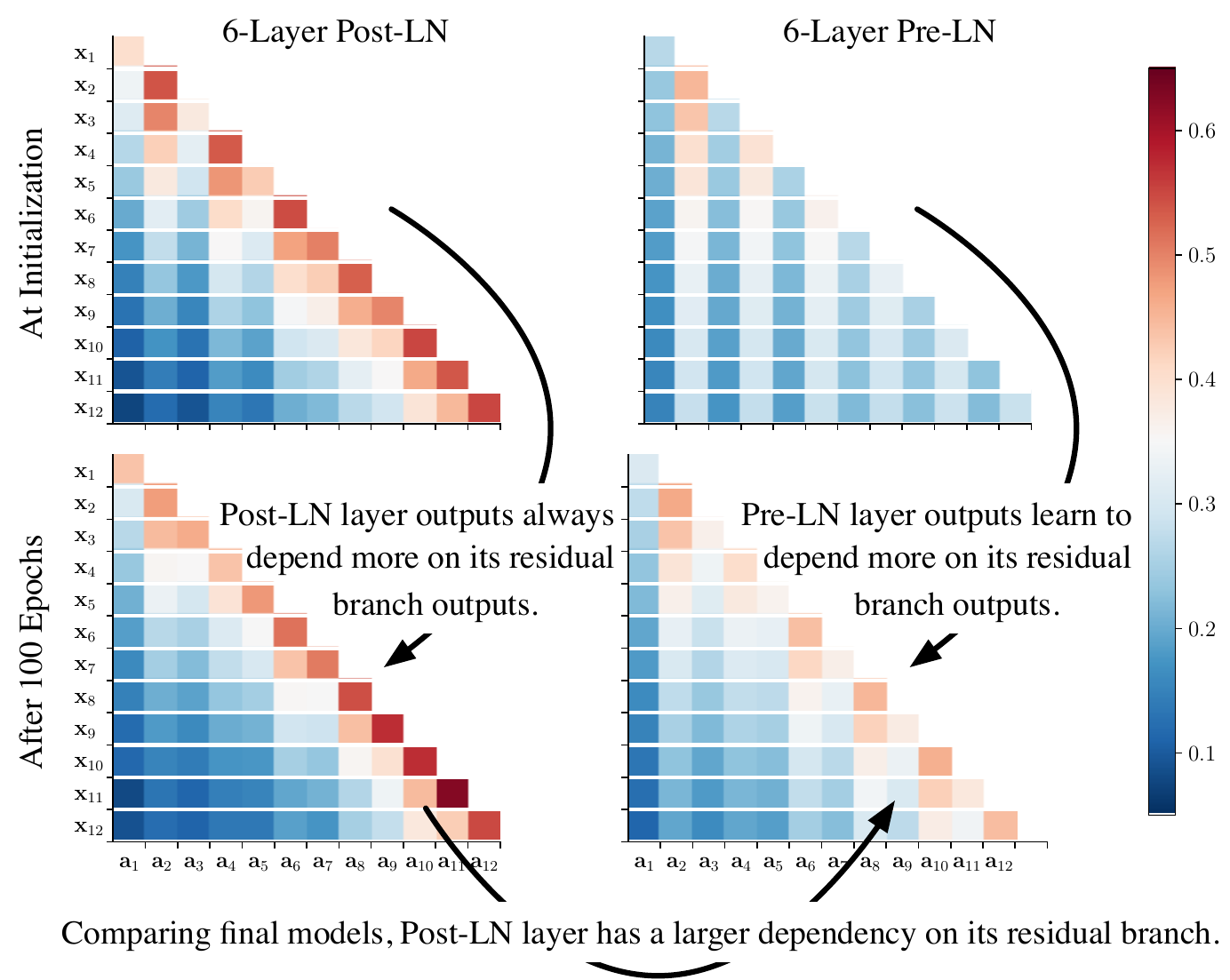}
\caption{
$\beta_{i, j}$ in 6-Layer Post-LN and Pre-LN on the WMT-14 En-De dataset (contains 12 sub-layers).
}
\label{fig:layer-dependency}
\end{figure}

Intuitively, it is harder for Pre-LN layers to depend too much on their own residual branches.
In Pre-LN, layer outputs (\ie, $\xb^{(p\cdot)}_i$) are not normalized, and their variances are likely to be larger for higher layers\footnote{If $\ab_0$ and $\ab_1$ are independent, $\Var[\ab_0 + \ab_1] = \Var[\ab_0] + \Var[\ab_1]$; also, in our experiments $\Var[\xb^{(p\cdot)}_i]$ increases as $i$ becomes larger}.
Since $\beta_{i, i} = \frac{\sqrt{\Var[\ab_i]}}{\sqrt{\Var[\xb^{(p\cdot)}_{i-1} +\ab_i]}}$, $\beta_{i, i}$ is likely to be smaller for higher layers, which restricts $i$-th layer outputs from depending too much on its residual branch and inhibits the network from reaching its full potential.
In other words, Pre-LN restricts the network from being too deep 
(\ie, if it is hard to distinguish $\xb^{(p\cdot)}_i$ and $\xb^{(p\cdot)}_{i+1}$,
appending one layer would be similar to doubling the width of the last layer),
while Post-LN gives the network the choice of being wider or deeper.

\subsection{Amplification Effect at Initialization}

Although depending more on residual branches allows the model to have a larger potential, it amplifies the fluctuation brought by parameter changes. 
For a network $\hxb = \cF(\xb_0, W)$ where $\xb_0$ is the model input and $W$ is the parameter, 
the output change caused by parameter perturbations is $\Var[\cF(\xb_0, W) - \cF(\xb_0, W^*)]$, where $W^* = W + \delta$. 
Its relationship with $N$ is described in Theorem~$\ref{theo:layer-dependency-and-shift}$, and the derivation is elaborated in Appendix~\ref{appendix:theo2}. 
\begin{theorem}
Consider a $N$-layer Transformer $\hxb = \cF(\hxb_0, W)$ at initialization, where $\hxb_0$ is the input and $W$ is the parameter. 
If the layer dependency stays the same after a parameter change 
(\ie, $\beta_{i, j}$ has the same value after changing $W$ to $W^*$, where $W$ is randomly initialized and $\delta = W^* - W$ is independent to $W$), the output change (\ie, $\Var[\cF(\xb_0, W) - \cF(\xb_0, W^*)]$) can be estimated as $\sum_{i=1}^N \beta^2_{i, i} C$ where $C$ is a constant.
\label{theo:layer-dependency-and-shift}
\end{theorem}
If $\Var[\ab_i]$ is the same for all layers, Pre-LN sets $\beta^2_{i, i}$ as $1/i$, and Post-LN sets $\beta^2_{i, i}$ as a constant.
Thus, we have Corollary~\ref{cor:pre-ln} and \ref{cor:post-ln} as below.
\begin{corollary}
For a $N$-layer Pre-LN $\cF$, we have $\Var[\cF(\xb_0, W) - \cF(\xb_0, W^*)] = O(\log N)$.
\label{cor:pre-ln}
\end{corollary}
\begin{corollary}
For a $N$-layer Post-LN $\cF$, we have $\Var[\cF(\xb_0, W) - \cF(\xb_0, W^*)] = O(N)$.
\label{cor:post-ln}
\end{corollary}
They show that, since Post-LN relies more on residual branches than Pre-LN (\ie, has a larger $\beta^2_{i, i}$), the perturbation is amplified to a larger magnitude. 
To empirically verify these relationships, we calculate $|\cF(\xb_0, W) - \cF(\xb_0, W^*)|^2_2$ for Pre-LN and Post-LN and visualize the results in Figure~\ref{fig:stability}. 
In Corollary~\ref{cor:post-ln}, $N$ is linearly associated with  $|\mathcal{F} - \mathcal{F}^*|_2^2$ for Post-LN; and in Corollary~\ref{cor:pre-ln}, $\log N$ is linearly associated with 
$|\mathcal{F} - \mathcal{F}^*|_2^2$ for Pre-LN.
These relationships match the observation in our experiments (as in Figure~\ref{fig:stability}).
For further verification, we measure their correlation magnitudes by $R^2$ and find $R^2 = 0.99$ in both cases. 

Moreover, we replace the random noise $\delta$ with optimization updates (\ie, setting $W^* = W + {\mbox{Adam}}(\Delta W)$, where ${\mbox{opt}}(\cdot)$ is update calculated by the Adam optimizer) and visualize output shifts. 
This replacement makes the correlation between $|\mathcal{F} - \mathcal{F}^*|_2^2$ and $N$ (for Post-LN) or $\log N$ (for Pre-LN) to be weaker (\ie, $R^2 = 0.75$).
Still, as in Figure~\ref{fig:stability}, the output shift $|\mathcal{F} - \mathcal{F}^*|_2^2$ for Post-LN is larger than Pre-LN by multiple magnitudes. 

Intuitively, large output shifts would destabilize the training~\cite{Li2018VisualizingTL}.
Also, as elaborated in Appendix~\ref{appendix:theo2}, the constant $C$ in Theorem~\ref{theo:layer-dependency-and-shift} is related to network derivatives and would be smaller as training advances, which explains why warmup is also helpful for the standard SGD. 
Therefore, we conjecture it is the large output shift of Post-LN results in unstable training.
We proceed to stabilize Post-LN by controlling the dependency on residual branches in the early stage of training.

\subsection{\textit{Admin} -- Adaptive Model Initialization}
\label{subsec:admin}

In light of our analysis, we add additional parameters (\ie, $\bomega$) to control residual dependencies of Post-LN and stabilize training by adaptively initializing  $\bomega$
to ensure an $O(\log N)$ output change. 

Due to different training configurations and model specificities (\eg, different models may use different activation functions and dropout ratios), it is hard to derive a universal initialization method. 
Instead, we decompose model initialization into two phrases: \textit{Profiling} and \textit{Initialization}. 
Specifically, \our adds new parameters $\bomega$ and constructs its i-th sub-layer as $\xb_{i} = \LN (\bbb_i)$, where $\bbb_i=\xb_{i-1} \cdot \bomega_i + f_i (\xb_{i-1})$, $\bomega_i$ is a $D$-dimension vector and $\cdot$ is element-wise product. 
Then the \textit{Profiling} phrase and \textit{Initialization} phrase are:

\smallsection{Profiling}
After initializing the network with a standard method (initializing $\bomega_i$ as $\mathbf{1}$), conduct forward propagation without parameter updating and record the output variance of residual branches (\ie, calculate $\Var[f_i(\xb_{i-1})]$). 
Since all elements in the same parameter/output matrix are independent to each other and are subject to the same distribution, it is sufficient to use a small number of instances in this phrase.
In our experiments, the first batch (no more than 8192 tokens) is used. 

\smallsection{Initialization}
Set $\bomega_i = \sqrt{\sum_{j < i} \Var[f_j(\xb_{j-1})]}$ and initialize all other parameters with the same method used in the \textit{Profiling} phrase. 

In the early stage, \our sets $\beta^2_{i, i}$ to approximately $\frac{1}{i}$ and ensures an $O(\log N)$ output change, thus stabilizing training. 
Model training would become more stable in the late stage (the constant $C$ in Theorem~\ref{theo:layer-dependency-and-shift} is related to parameter gradients), and each layer has the flexibility to adjust $\bomega$ and depends more on its residual branch to calculate the layer outputs. 
After training finishes, \our can be reparameterized as the conventional Post-LN structure (\ie, removing $\bomega$).  
More implementation details are elaborated in Appendix~\ref{appendix:implement}.

\begin{figure}[t]
\centering
\includegraphics[width=\linewidth]{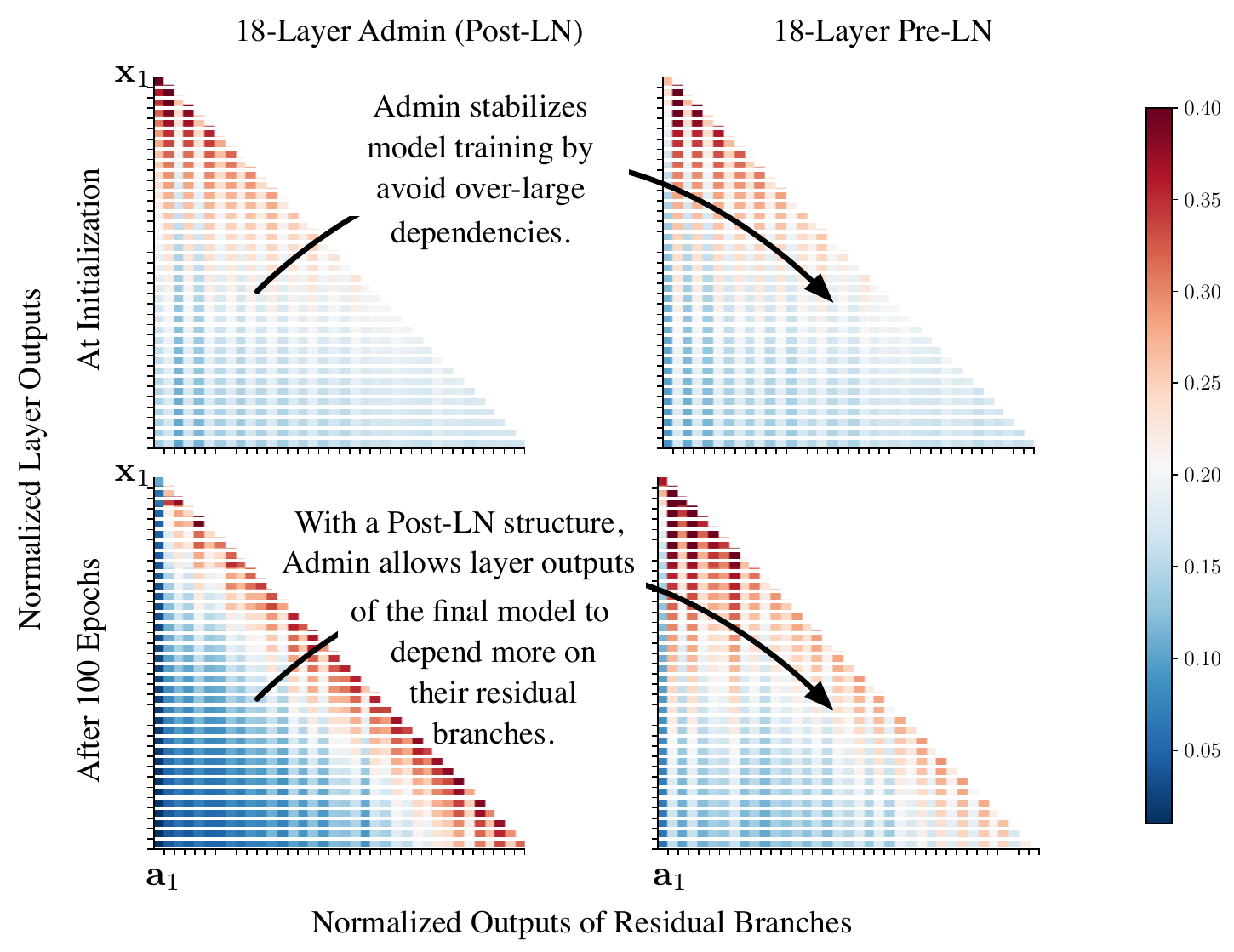}
\caption{
$\beta_{i, j}$ of 18-Layer Admin (Post-LN) and Pre-LN on the WMT-14 En-De dataset.
}
\label{fig:large-layer-dependency}
\end{figure}

\begin{table*}[t]
    \centering
    \caption{BLEU on IWSLT'14 De-En and WMT'14 En-Fr/De (AL-BL refers A-layer encoder \& B-layer decoder).}
    \label{tab:wmt14de}
    \begin{tabularx}{0.95\linewidth}{l|c|YY|YYY}
        \toprule
        Dataset &
        IWSLT'14 De-En &
        \multicolumn{2}{c|}{WMT'14 En-Fr} &
        \multicolumn{3}{c}{WMT'14 En-De}\\
        \midrule
        Enc \#--Dec \#
        & 6L--6L (small) 
        & 6L--6L & 60L--12L
        & 6L--6L & 12L--12L & 18L--18L \\

        \midrule
        Post-LN & 35.64$\pm$0.23
        & 41.29 & failed 
        & 27.80 & failed & failed \\
        Pre-LN & 35.50$\pm$0.04
        & 40.74 & 43.10
        & 27.27 & 28.26 & 28.38 \\ 
        Admin & \textbf{35.67$\pm$0.15} 
        & \textbf{41.47} & \textbf{43.80}
        & \textbf{27.90} & \textbf{28.58} & \textbf{29.03} \\
        \bottomrule
    \end{tabularx}
\end{table*}

To verify our intuition, we calculate the layer dependency of 18-Layer models and visualize the result in Figure~\ref{fig:large-layer-dependency}. 
Figures~\ref{fig:layer-dependency} and \ref{fig:large-layer-dependency}
show that \our avoids over-large dependencies at initialization and unleashes the potential to make the layer outputs depend more on their residual outputs in the final model. 
Moreover, we visualize the output change of \our in Figure~\ref{fig:stability}. 
Benefiting from the adaptive initialization, the output change of \our gets roughly the same increase speed as Pre-LN, even constructed in the Post-LN manner. 
Also, although \our is formulated in a Post-LN manner and suffers from gradient vanishing, 18-layer \our successfully converges and outperforms 18-layer Pre-LN (as in Table~\ref{tab:wmt14de}). 
This evidence supports our intuition that the large dependency on residual branches amplifies the output fluctuation and destabilizes training. 


\section{Experiments}

We conduct experiments on 
IWSLT'14 De-En, WMT'14 En-De, and WMT'14 En-Fr. 
More details are elaborated in Appendix~\ref{appendix:exp}. 

\subsection{Performance Comparison}

We use BLEU as the evaluation matric and summarize the model performance in Table~\ref{tab:wmt14de}. 
On the WMT'14 dataset, we use Transformer-base models with 6, 12, or 18 layers. 
\our achieves a better performance than Post-LN and Pre-LN in all three settings. 
Specifically, 12-Layer and 18-Layer Post-LN diverges without the adaptive initialization.
Pre-LN converges in all settings, but it results in sub-optimal performance. 
Admin not only stabilizes the training of deeper models but benefits more from the increased model capacity then Pre-LN, which verifies our intuition that the Pre-LN structure limits the model potential. 
As in Figure~\ref{fig:wmt14ende_devppl_6_18} and Figure~\ref{fig:wmt14-devppl-12l-iwslt}, although the 6-layer Pre-LN converges faster than Post-LN, its final performance is worse than Post-LN. 
In contrast, \our not only achieves the same convergence speed with Pre-LN in the early stage but reaches a good performance in the late stage. 

We use 6-layer Transformer-small (its hidden dimension is smaller than the base model) on the IWSLT'14 dataset, and all methods perform similarly.
Still, as in Figure~\ref{fig:grid-search}, Admin outperforms the other two by a small margin. 
Together with WMT'14 results, it implies the training stability is related to layer number. 
For shallow networks, the stability difference between Post-LN and Pre-LN is not significant (as in Figure~\ref{fig:stability}), and all methods reach reasonable performance. 
It is worth mentioning that attention and activation dropouts have an enormous impact on IWSLT'14, which is smaller than WMT'14 datasets. 

\begin{figure}[t]
\centering
\includegraphics[width=\linewidth]{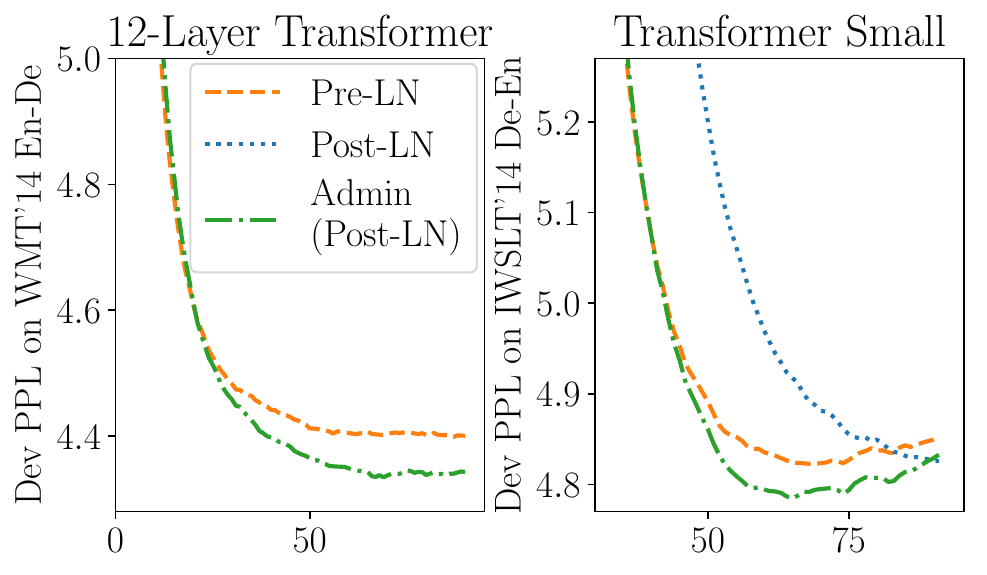}
\caption{
Development PPL on the WMT'14 En-De dataset and the IWLST'14 De-En dataset. }
\label{fig:wmt14-devppl-12l-iwslt}
\end{figure}

To further explore the potential of \our, we train Transformers with a larger size. 
Specifically, we expand the Transformer-base configuration to have a 60-layer encoder and a 12-layer decoder. 
As in Table~\ref{tab:wmt14de}, our method achieves a BLEU score of 43.8 on the WMT'14 En-Fr dataset, the new state-of-the-art without using additional annotations (\eg, back-translation).
More discussions are conducted in Appendix~\ref{appendix:enfr} to compare this model with the current state of the art. 
Furthermore, in-depth analyses are summarized in \citet{Liu2020VeryDT}, including systematic evaluations on the model performance (with TER, METEOR, and BLEU), comprehensive discussions on model dimensions (\ie, depth, head number, and hidden dimension), and fine-grained error analysis.
It is worth mentioning that the 60L-12L Admin model achieves a 30.1 BLEU score on WMT'14 En-De~\cite{Liu2020VeryDT}.

\subsection{Connection to Warmup}
\label{subsec:warmup}

Our previous work~\cite{Liu2019OnTV} establishes that the need for warmup comes from the unstable adaptive learning rates in the early stage.
Still, removing the warmup phrase results in more severe consequences for Transformers than other architectures.
Also, warmup has been found to be useful for the vanilla SGD~\cite{Xiong2019OnLN}.

Theorem~\ref{theo:post-ln-encoder-gradient} establishes that
$
\Var[\cF(\xb_0, W) - \cF(\xb_0, W^*)] \approx \sum_{i=1}^N \beta^2_{i, i} C
$
where $C=\Var[\cG_i(\hxb^*_{i-1}, W_i) - \cG_i(\hxb^*_{i-1}, W_i^*)]$. 
In the early stage of training, the network has larger parameter gradients and thus larger $C$. 
Therefore, using a small learning rate at initialization helps to alleviate the massive output shift of Post-LN. 
We further conduct experiments to explore whether more prolonged warmups can make up the stability difference between Post-LN and Pre-LN.
We observe that 18-layer Post-LN training still fails after extending the warmup phrase from 8 thousand updates to 16, 24, and 32 thousand.
It shows that learning rate warmup alone cannot neutralize the instability of Post-LN. 
Intuitively, massive output shifts not only require a small learning rate but also unsmoothes the loss surface~\cite{Li2018VisualizingTL} and make the training ill-conditioned.

\our regularizes the model behavior at initialization and stabilizes the training. 
To explore whether \our is able to stabilize the training alone, we remove the warmup phase and conduct a grid search on optimizer hyper-parameters. 
The results are visualized in Figure~\ref{fig:grid-search}. 
It shows that as Post-LN is more sensitive to the choice of hyper-parameters, \our successfully stabilizes the training without hurting its potential.  

\begin{figure}[t]
\centering
\includegraphics[width=\linewidth]{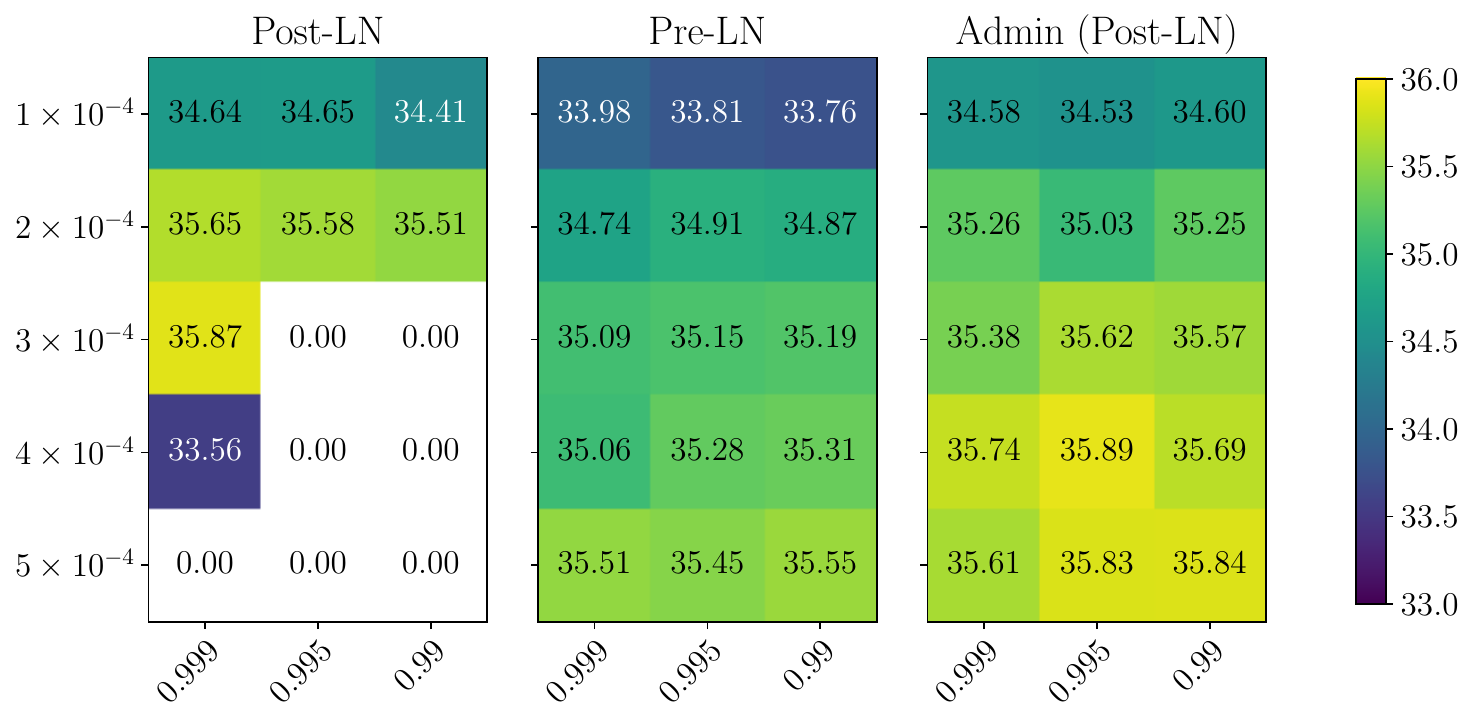}
\caption{
BLEU score of Post-LN, Pre-LN and Admin on the IWSLT'14 De-En dataset (x-axis is the $\beta_2$ for adaptive optimizers and y-axis is the learning rate). Pre-LN converges in all settings while Post-LN diverges in 7 out of 15 settings. 
When Post-LN converges, it outperforms Pre-LN in 7 out of 8 settings. 
Admin stabilizes Post-LN training and outperforms Pre-LN (its best performance is comparable with Post-LN). }
\label{fig:grid-search}
\end{figure}

\subsection{Comparing to Other Initializations}

We compare our methods with three initialization methods, \ie, ReZero~\cite{Bachlechner2020ReZeroIA}, FixUp~\cite{zhang2018residual}, and LookLinear~\cite{balduzzi17b}.
Specifically, we first conduct experiments with 18-layer Transformers on the WMT'14 De-En dataset. 
In our experiments, we observe that all of ReZero (which does not contain layer normalization), FixUp (which also does not contain layer normalization), and LookLinear (which is incorporated with Post-LN) leads to divergent training.
With further analysis, we find that the half-precision training and dropout could destabilize FixUp and ReZero, due to the lack of layer normalization.
Simultaneously, we find that even for shadow networks, having an over small reliance on residual branches hurts the model performance, which also supports our intuition. 
For example, as elaborated in Appendix~\ref{appendix:rezero}, applying ReZero to Transformer-small leads to a 1-2 BLEU score drop on the IWSLT'14 De-En dataset. 


\section{Related Work}
\label{sec:related}

\smallsection{Transformer}
Transformer~\cite{Vaswani2017AttentionIA} has led to a series of breakthroughs in various domains~\cite{Devlin2019BERTPO,Velickovic2017GraphAN,Huang2019MusicTG,Parmar2018ImageT,Ramachandran2019StandAloneSI}. 
\citet{Liu2019OnTV} show that compared to other architectures, removing the warmup phase is more damaging for Transformers, especially Post-LN. 
Similarly, it has been found that the original Transformer (referred to as Post-LN) is less robust than its Pre-LN variant~\cite{Baevski2018AdaptiveIR,Nguyen2019TransformersWT,wang-etal-2019-learning}. 
Our studies go beyond the existing literature on gradient vanishing~\cite{Xiong2019OnLN} and identify an essential factor influencing Transformer training greatly. 

\smallsection{Deep Network Initialization} 
It has been observed that deeper networks can lead to better performance. 
For example, \citet{Dong2020TowardsAR} find that the network depth players a similar role with the sample number in numerical ODE solvers, which hinders the system from getting more precise results.
Many attempts have been made to clear obstacles for training deep networks, including various initialization methods.
Based on the independence among initialized parameters, one method is derived and found to be useful to handle the gradient vanishing~\cite{Glorot2010UnderstandingTD}. 
Similar methods are further developed for ReLU networks~\cite{He2015DelvingDI}. 
\citet{He2016DeepRL} find that deep network training is still hard even after addressing the gradient vanishing issue and propose residual networks.
\citet{David2017Shattered} identifies the shattered gradient issue and proposes LookLinear initialization. 

On the other hand, 
although it is observed that scaling residual outputs to smaller values helps to stabilize training~\cite{Hanin2018HowTS,Mishkin2015AllYN,zhang2018residual,Bachlechner2020ReZeroIA,Goyal2017AccurateLM}, there is no systematic analysis on what complicates Transformer training or its underlying connection to the dependency on residual branches.
Here, we identify that unbalanced gradients are not the direct cause of the Post-LN instability, recognize the amplification effect, and propose a novel adaptive initialization method.


\section{Conclusion}

In this paper, we study the difficulties of training Transformers in theoretical and empirical manners. 
Our study in Section~\ref{sec:gradient-distribution} suggests that the gradient vanishing problem is not the root cause of unstable Transformer training. 
Also, the unbalanced gradient distribution issue is mostly addressed by adaptive optimizers. 
In Section~\ref{sec:layer-dependency}, we reveal the root cause of the instability to be the strong dependency on residual branches, which amplifies the fluctuation caused by parameter changes and destabilizes model training. 
In light of our analysis, we propose \our, an adaptive initialization method to stabilize Transformers training. 
It controls the dependency at the beginning of training and maintains the flexibility to capture those dependencies once training stabilizes. 
Extensive experiments verify our intuitions and show that, without introducing additional hyper-parameters, \our achieves more stable training, faster convergence, and better performance. 

Our work opens up new possibilities to not only further push the state-of-the-art but understand deep network training better. 
It leads to many interesting future works, including generalizing Theorem~\ref{theo:layer-dependency-and-shift} to other models, designing new algorithms to automatically adapt deep networks to different training configurations, upgrading the Transformer architecture, and applying our proposed \our to conduct training in a larger scale.


\section*{Acknowledge}
We thank all reviewers for their constructive comments; Chengyu Dong, Haoming Jiang, Jingbo Shang, Xiaotao Gu, and Zihan Wang for valuable discussions and comments; Jingbo Shang for sharing GPU machines; and Microsoft for setting up GPU machines.
The research was sponsored in part by DARPA No. W911NF-17-C-0099 and No.  FA8750-19-2-1004, National Science Foundation IIS-19-56151, IIS-17-41317, IIS 17-04532, and IIS 16-18481, and DTRA HDTRA11810026.

\bibliography{citation.bib}
\bibliographystyle{acl_natbib}

\onecolumn
\appendix
\appendixpage

\setcounter{theorem}{0}    

\section{Gradients at Initialization}
\label{appendix:ini_grad}

Here, we first reveal that Pre-LN does not suffer from the gradient vanishing. 
Then we establish that only the Post-LN decoder suffers from the gradient vanishing, but not the Post-LN encoder. 
For simplicity, we use $\Delta \xb$ to denote gradients, \ie, $\Delta \xb = \frac{\partial \mathcal{L}}{\partial \xb}$ where $\mathcal{L}$ is the training objective. 
Following the previous study~\cite{Bengio1994LearningLD,Glorot2010UnderstandingTD,He2015DelvingDI,saxe2013exact}, we analyze the gradient distribution at the very beginning of training, assume that the randomly initialized parameters and the partial derivative with regard to module inputs are independent.

\subsection{Pre-LN Analysis}
\label{subsec:preln-analysis}
For Pre-LN encoders, we have 
$ \xb\ep_{2i} = \xb\ep_{2i-1} + \FFN(\LN(\xb\ep_{2i-1}))$ and
$ \Delta\xb\ep_{2i-1} = \Delta\xb\ep_{2i} (1 + \frac{\partial \FFN(\LN(\xb\ep_{2i-1})) }{\partial \xb\ep_{2i}})$. 
At initialization, the two terms on the right part are approximately independent and $E[\frac{\partial \FFN(\LN(\zb\ep_{2i-1})) }{\partial \xb\ep_{2i}}] = 0$. Therefore we have $ \Var[\Delta\xb\ep_{2i-1}] \geq \Var[\Delta\xb\ep_{2i}]$.
Similarly, we can get $ \Var[\Delta\xb\ep_{2i-2}] \geq \Var[\Delta\xb\ep_{2i-1}]$ thus $\forall i \leq j, \Var[\Delta\xb\ep_{i}] \geq \Var[\Delta\xb\ep_{j}]$. 
Applying the same analysis to Pre-LN decoders, we can get $\forall i \leq j, \Var[\Delta\xb\pd_{i}] \geq \Var[\Delta\xb\pd_{j}]$.
Thus, lower layers have larger gradients than higher layers, and gradients do not vanish in the backpropagation. 
\begin{remark}
For Pre-LN, if $\forall i, \Delta \xb_i^{(p\cdot)}$ and the derivatives of modules in the $i$-th sub-layer are independent, then $\forall i \leq j, \Var[\Delta \xb^{(p\cdot)}_{i}] \geq \Var[\Delta \xb^{(p\cdot)}_{j}]$.
\label{remark: pre-ln-gradient}
\end{remark}

\subsection{Post-LN Encoder Analysis}
\label{subsec:postln-encoder-analysis}
Different from Pre-LN, $\xb_{i}^{(oe)}$ and $\xb_{i-1}^{(oe)}$ are associated with not only the residual connection but the layer normalization, which makes it harder to establish the connection on their gradients.
After making assumptions on the model initialization, we find that lower layers in Post-LN encoder also have larger gradients than higher layers, and gradients do not vanish in the backpropagation through the encoder.   

\begin{theorem}
For Post-LN Encoders, if $\bgamma$ and $\bnu$ in the Layer Norm are initialized as $1$ and $0$ respectively; all other parameters are initialized by symmetric distributions with zero mean; $\xb_{i}^{(oe)}$ and $\Delta \xb_i^{(oe)}$ are subject to symmetric distributions with zero mean; the variance of $\xb_{i}^{(oe)}$ is $1$ (\ie, normalized by Layer Norm); $\Delta \xb_i^{(oe)}$ and the derivatives of modules in $i$-th sub-layer are independent, we have $\Var[\Delta \xb_{i-1}] \geq \Var[\Delta \xb_{i}]$.
\end{theorem}
\begin{proof}
We first prove $\Var[\Delta \xb\eo_{2i-1}] \geq \Var[\Delta \xb\eo_{2i}]$, \ie, the backpropagation through FFN sublayers does not suffer from gradient vanishing. 
In Post-LN encoders, the output of FFN sublayers is calculated as $\xb\eo_{2i} = \LN(\bbb\eo_{2i})$ where $\bbb\eo_{2i} =
\xb\eo_{2i-1} + \max(0, \xb\eo_{2i-1}\wf)\wff$.
Since at initialization, $\wf$ and $\wff$ are independently randomized by symmetric distributions, we have $\E[\bbb\eo_{2i}] = 0$ and 
$$
\xb\eo_{2i} = \frac{\xb\eo_{2i-1} + \max(\xb\eo_{2i-1}\wf, 0)\wff}{\sigma_{b,2i}}
$$ 
where $\sigma_{b,2i}^2 = \Var[\bbb\eo_{2i}]$. 
Referring to the dimension of $\wf$ as $D \times D_f$, \citet{He2015DelvingDI} establishes that
$$
\Var[\max(\xb\eo_{2i-1}\wf, 0)\wff] = \frac{1}{2} D D_f \Var[w^{(1)}] \Var[w^{(2)}] \Var[\xb\eo_{2i-1}].
$$
Since in Post-LN, $\xb\eo_{2i-1}$ is the output of layer norm, we have $\Var[\xb\eo_{2i-1}] = 1$. 
Thus, 
\begin{align}
\sigma_{b,2i}^2 & = \Var[\bbb\eo_{2i}] = \Var[\xb\eo_{2i-1}] + \Var[\max(\xb\eo_{2i-1}\wf, 0)\wff] \nonumber\\
& = 1 + \frac{1}{2} D D_f \Var[w^{(1)}] \Var[w^{(2)}].
\label{eqn:ffn-sigma}
\end{align}
Assuming different terms are also independent in the backpropagation, we have
$$
\Var[\Delta \xb\eo_{2i-1}] \geq \Var[\frac{1}{\sigma_{b, 2i}} (\Delta \xb\eo_{2i} + \Delta \xb\eo_{2i} \frac{\partial \max(\xb\eo_{2i-1}\wf, 0)\wff}{\partial \xb\eo_{2i-1}})].
$$ 
At initialization, \citet{He2015DelvingDI} establishes that
$$
\Var[\Delta \xb\eo_{2i} \frac{\partial \max(\xb\eo_{2i-1}\wf, 0)\wff}{\partial \xb\eo_{2i-1}}] = \frac{1}{2} D D_f \Var[w^{(1)}] \Var[w^{(2)}] \Var[\Delta \xb\eo_{2i}].
$$
Therefore, we have
\begin{align}
\Var[\Delta \xb\eo_{2i-1}]& \geq  \frac{1}{\sigma^2_{b, 2i}} (1 + \frac{1}{2} D D_f \Var[w^{(1)}] \Var[w^{(2)}]) \Var[\Delta \xb\eo_{2i}].
\label{eqn:ffn-back-propagation}
\end{align}
Combining Equation~\ref{eqn:ffn-sigma} with Equation~\ref{eqn:ffn-back-propagation}, we have
\begin{align}
\Var[\Delta \xb\eo_{2i-1}] \geq \Var[\Delta \xb\eo_{2i}]
\label{eqn:ffn-no-vanish}
\end{align}
which shows the backpropagation through FFN sublayers does not suffer from gradient vanishing. 

Now we proceed to prove that, $\Var[\Delta \xb\eo_{2i-2}] \geq \Var[\Delta \xb\eo_{2i-1}]$, \ie, the backpropagation through Self-Attention sublayers do not suffer from gradient vanishing. 
In Post-LN encoders, the output of Self-Attention sublayers are calculated as 
$\xb\eo_{2i-1} = \LN(\bbb\eo_{2i-1})$ where $\bbb\eo_{2i-1} = \xb\eo_{2i-2} + \ab\eo_{2i-1}$ and $\ab\od_{2i-1} = \sum_{h} \softmax(\xb\eo_{2i-2} \wq_h \wk_h \xbteo_{2i-2})\xb\eo_{2i-2} \wv_h \wo_h$. 
At initialization, since $\wq$, $\wk$, $\wv$, and $\wo$ are independently randomized by symmetric distributions, we have $\E[\bbb\od_{2i-1}] = 0$, thus $\xb\eo_{2i-1} = \frac{\bbb\eo_{2i-1}}{\sigma_{b, 2i-1}}$, where $\sigma^2_{b, 2i-1} = \Var[\bbb\eo_{2i-1}] =  \Var[\xb\eo_{2i-2}] + \Var[\ab\eo_{2i-1}]$.

Referring $\E[\softmax^2(\xb\eo_{2i-2} \wq_h \wk_h \xbteo_{2i-2})]$ as $P_h$, we have 
$$
\Var[\ab\od_{2i-1}] = \Var[\xb\eo_{2i-2} \wv_h \wo_h] H P_h.
$$
Similar to \citet{He2015DelvingDI}, we have
$$
\Var[\xb\eo_{2i-2} \wv_h \wo_h] = \frac{D^2}{H} \Var[\xb\eo_{2i-2}] \Var[w^{(V_1)}] \Var[w^{(V_2)}].
$$
Since $\xb\eo_{2i-2}$ is the output of layer norm, we have $\Var[\xb\eo_{2i-2}] = 1$. Thus, 
\begin{align}
\sigma^2_{b, 2i-1} = 1 + D^2 P_h \Var[\xb\eo_{2i-2}] \Var[w^{(V_1)}] \Var[w^{(V_2)}].
\label{eqn:satt-sigma}
\end{align}
In the backpropagation, we have
\begin{align*}
& \Var[\Delta \xb\eo_{2i-2}] \geq \Var[\frac{1}{\sigma_{b, 2i-1}} (\Delta \xb\eo_{2i-1} + \Delta \xb\eo_{2i-1} \sum_h \frac{\partial\softmax(\xb\eo_{2i-2} \wq_h \wk_h \xbteo_{2i-2})\xb\eo_{2i-2} \wv_h \wo_h}{\partial \xb\eo_{2i-2}})] \\
&\geq \frac{1}{\sigma^2_{b, 2i-1}}(\Var[\Delta \xb\eo_{2i-1}] + \Var[\Delta \xb\eo_{2i-1} \sum_h \softmax(\xb\eo_{2i-2} \wq_h \wk_h \xbteo_{2i-2})\frac{\partial\xb\eo_{2i-2} \wv_h \wo_h}{\partial \xb\eo_{2i-2}} ])
\end{align*}
At initialization, we assume $\Delta \xb\eo_{2i-1}$ and model parameters are independent~\cite{He2015DelvingDI}, thus
\begin{align*}
& \Var[\Delta \xb\eo_{2i-1} \sum_h \softmax(\xb\eo_{2i-2} \wq_h \wk_h \xbteo_{2i-2})\frac{\partial\xb\eo_{2i-2} \wv_h \wo_h}{\partial \xb\eo_{2i-2}}] \\
=&  D^2 P_h \Var[\Delta \xb\eo_{2i-1}] \Var[w^{(V_1)}] \Var[w^{(V_2)}]
\end{align*}
Therefore, we have
\begin{align}
\Var[\Delta \xb\eo_{2i-2}]& \geq  \frac{1}{\sigma^2_{b, 2i-1}} (1 + D^2 P_h \Var[w^{(V_1)}] \Var[w^{(V_2)}]) \Var[\Delta \xb\eo_{2i-1}].
\label{eqn:satt-back-propagation}
\end{align}
Integrating Equation~\ref{eqn:satt-sigma} with Equation~\ref{eqn:satt-back-propagation}, we have
\begin{align}
\Var[\Delta \xb\eo_{2i-2}] \geq \Var[\Delta \xb\eo_{2i-1}].
\label{eqn:satt-no-vanish}
\end{align}
Combining Equation~\ref{eqn:ffn-no-vanish} and Equation~\ref{eqn:satt-no-vanish}, we have $\Var[\Delta \xb_{i-1}] \geq \Var[\Delta \xb_{i}]$.
\end{proof}


\subsection{Post-LN Decoder Analysis}
\label{subsec:postln_decoder_analysis}
In Post-LN, the Encoder-Attention sub-layer suffers from gradient vanishing. 
The Encoder-Attention sub-layer calculates outputs as 
$\xb\od_{3i-1} = \LN(\bbb\od_{3i-1})$ where $\bbb\od_{3i-1} = \xb\od_{3i-2} + \ab\od_{3i-1}$ and $\ab\od_{3i-1} = \sum_{h} \softmax(\xb\od_{3i-2} \wq_h \wk_h {\xb^T}^{(oe)})\xb^{(oe)} \wv_h \wo_h$. 
Here $\xb^{(oe)}$ is encoder outputs and $\softmax$ is the row-wise softmax function. 
In the backpropagation, $\Delta \xb\od_{3i-2} \approx \frac{\Delta \xb\od_{3i-1}}{\sigma_{b, 3i-1}} (1 + 
\frac{\partial \ab\od_{3i-1} }{\partial \xb\od_{3i-2}}).$
All of the backpropagations from $\ab\od_{3i-1} $ to $\xb\od_{3i-2}$ went through the softmax function, we have $\Var[\frac{\partial \ab\od_{3i-1} }{\partial \xb\od_{3i-2}}] + 1 \leq \sigma^2_{b, 3i-1}$. 
Thus, those backpropagations suffer from gradient vanishing. 
This observation is further verified in Figure~\ref{fig:gradient-histogram}, as the encoder attention bars (gradients of encoder-attention outputs) are always shorter than self-attention bars (gradients of encoder-attention inputs), while adjacent self-attention bars and fully connected bars usually have the same length.

\subsection{Distributes of Unbalanced Gradients}
\label{subsec:attention_gradients}

\begin{figure}[t]
\centering
\includegraphics[width=\linewidth]{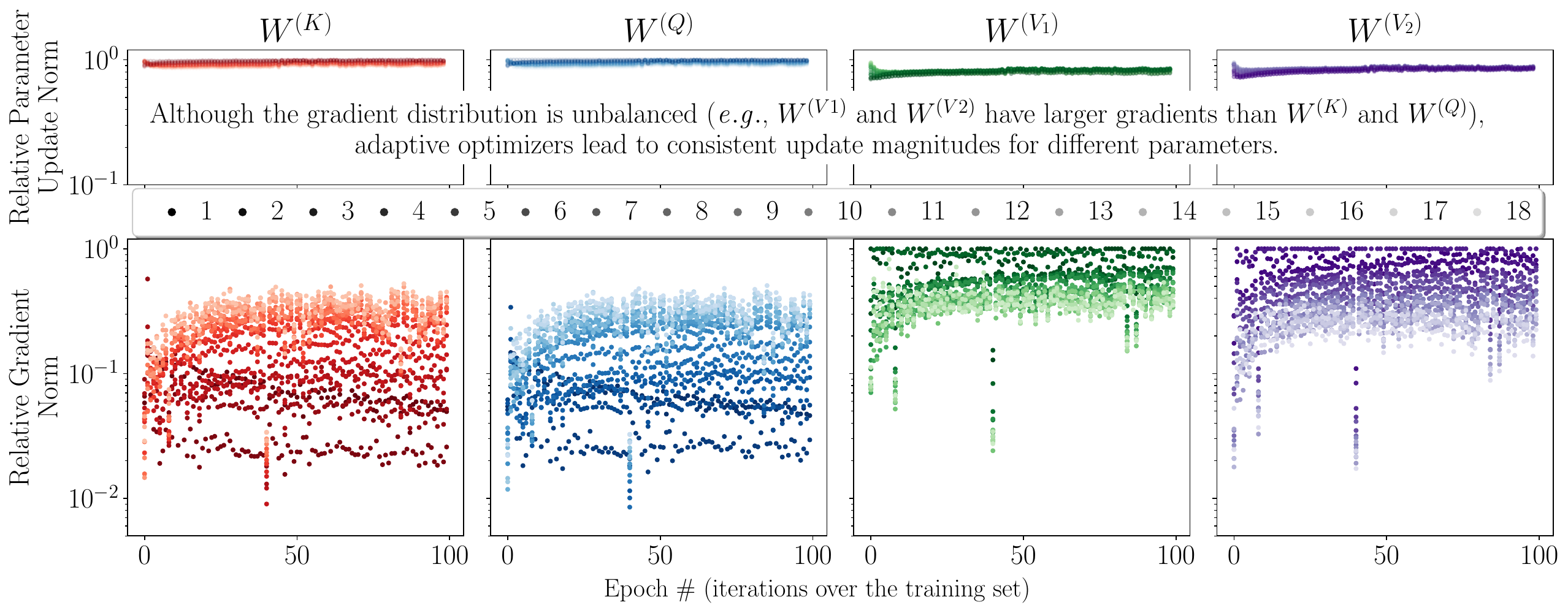}
\caption{
Relative Norm of Gradient ($\Delta W_i$, where $W_i$ is the checkpoint of $i$-th epoch) and Update ($|W_{i+1} - W_i|$) of Self-Attention Parameters in 12-Layer Pre-LN.
}
\label{fig:update-with-gradient-four}
\end{figure}

As in Figure~\ref{fig:all_histogram} and Figure~\ref{fig:update-with-gradient-four}, the gradient distribution of Attention modules is unbalanced even for Pre-LN. 
Specifically, parameters within the softmax function (\ie, $\wk$ and $\wv$) suffer from gradient vanishing (\ie, $\frac{\partial \softmax(x_0, \cdots, x_i, \cdots)}{\partial x_i} \leq 1$) and have smaller gradients than other parameters.

With further analysis, we find it is hard to neutralize the gradient vanishing of softmax. 
Unlike conventional non-linear functions like ReLU or sigmoid, softmax has a dynamic input length (\ie, for the sentences with different lengths, inputs of softmax have different dimensions). 
Although this setting allows Attention modules to handle sequential inputs, it restricts them from having stable and consistent backpropagation. 
Specifically, let us consider the comparison between softmax and sigmoid. 
For the sigmoid function, although its derivation is smaller than 1, this damping effect is consistent for all inputs. 
Thus, sigmoid can be neutralized by a larger initialization~\cite{Glorot2010UnderstandingTD}. 
For softmax, its damping effect is different for different inputs and cannot be neutralized by a static initialization. 

Also, we observe that adaptive optimizers largely address this issue. 
Specifically, we calculate the norm of parameter change in consequent epochs (\eg, $|\wk_{t+1} - \wk_{t}|$ where $\wk_{t}$ is the checkpoint saved after $t$ epochs) and visualize the relative norm (scaled by the largest value in the same network) in Figure~\ref{fig:update-with-gradient-four}.  
Comparing the relative norm of parameter gradients and parameter updates, we notice that: although the gradient distribution is unbalanced, adaptive optimizers successfully assign different learning rates to different parameters and lead to consistent update magnitudes. 
This result explains why the vanilla SGD fails for training Transformer (\ie, lacking the ability to handle unbalanced gradient distributions). 
Besides, it implies that the unbalanced gradient distribution (\eg, gradient vanishing) has been mostly addressed by adaptive optimizers and may not significantly impact the training instability. 

\section{Proof of Theorem~\ref{theo:layer-dependency-and-shift}}
\label{appendix:theo2}

Here, we elaborate the derivation for Theorem~\ref{theo:layer-dependency-and-shift}, which establishes the relationship between layer number and output fluctuation brought by parameter change. 

\begin{theorem}
Consider a $N$-layer Transformer $\hxb = \cF(\hxb_0, W)$, where $\hxb_0$ is the input and $W$ is the parameter. 
If the layer dependency stays the same after a parameter change 
(\ie, $\beta_{i, j}$ has the same value after changing $W$ to $W^*$, where $W$ is randomly initialized and $\delta = W^* - W$ is independent to $W$), the output change (\ie, $\Var[\cF(\xb_0, W) - \cF(\xb_0, W^*)]$) can be estimated as $\sum_{i=1}^N \beta^2_{i, i} C$ where $C$ is a constant.
\end{theorem}

\begin{proof}

We refer the module in $i$ sub-layer as $\ab_i = \cG_i(\hxb_{i-1}, W_i)$, where $\hxb_{i} = \sum_{j\leq i} \beta_{i, j} \hab_j$ is the normalized residual output and $\hab_i = \frac{\ab_i}{\sqrt{\Var{\ab_i}}}$ is the normalized module output. 
The final output is marked as $\hxb = \cF(\xb_0, W) = \sum_{j\leq N} \beta_{N, j} \hab_j$.
To simplify the notation, we use the superscript $*$ to indicate variables related to $W^*$, \eg, $\hxb^* = \cF(\xb_0, W^*)$ and $\ab_i^* = \cG_i(\hxb^*_{i-1}, W^*_i)$.

At initialization, all parameters are initialized independently.
Thus $\forall i \neq j$, $\hab_i$ and $\hab_j$ are independent and $1 = \Var[\sum_{j\le i} \beta_{i, j}\hab_j] = \sum_{j \le i} \beta_{i, j}^2$. 
Also, since $k$-layer and $(k+1)$-layer share the residual connection to previous layers, $\forall i, j \leq k$ we have $\frac{\beta_{i, k}}{\beta_{j, k}} = \frac{\beta_{i, k+1}}{\beta_{j, k+1}}$.
Thus $\forall i \leq k, \beta^2_{i, k + 1} = (1 - \beta^2_{k, k}) \beta^2_{i, k}$ and 
\begin{align}
\Var[\hxb_i - \hxb^*_i] & = \Var[\sum_{j \leq i} \beta_{i, j} (\hab_j - \hab_j^*)] = \sum_{j \leq i} \beta^2_{i, j} \Var[\hab_j - \hab_j^*] \nonumber\\
& = \beta^2_{i, i} \Var[\hab_i - \hab_i^*] + (1 - \beta^2_{i, i}) \Var[\hxb_i - \hxb^*_i]. \label{eqn:recurrent}
\end{align}

Now, we proceed to analyze $\Var[\hab_i - \hab_i^*]$. Specifically, we have
\begin{align}
\Var[\hab_i - \hab_i^*] &= \Var[\cG_i(\hxb_{i-1}, W_i) - \cG_i(\hxb^*_{i-1}, W_i^*)] \nonumber \\
&= \Var[\cG_i(\hxb_{i-1}, W_i) - \cG_i(\hxb^*_{i-1}, W_i) + \cG_i(\hxb^*_{i-1}, W_i^*) - \cG_i(\hxb^*_{i-1}, W_i^*)] \nonumber \\
&= \Var[\cG_i(\hxb_{i-1}, W_i) - \cG_i(\hxb^*_{i-1}, W_i)] + \Var[\cG_i(\hxb^*_{i-1}, W_i) - \cG_i(\hxb^*_{i-1}, W_i^*)]. 
\label{eqn:a-decompose}
\end{align}
Since $W$ is randomly initialized, $\Var[\cG_i(\hxb^*_{i-1}, W_i) - \cG_i(\hxb^*_{i-1}, W_i^*)]$ should have the same value for all layers, thus we use a constant $C$ to refer its value ($C=\Var[\cG_i(\hxb^*_{i-1}, W_i) - \cG_i(\hxb^*_{i-1}, W_i^*)]$ and $C \approx |\delta| \cdot |\nabla \cG_i(\hxb^*_{i-1}, W_i)|$). 
As to $\Var[\cG_i(\hxb_{i-1}, W_i) - \cG_i(\hxb^*_{i-1}, W_i)]$, since the sub-layer of Transformers are mostly using linear weights with ReLU nonlinearity and $1 = \Var[\cG_i(\hxb_{i-1}, W_i)] = \Var[\hxb_{i-1}]$, we have $\Var[\cG_i(\hxb_{i-1}, W_i) - \cG_i(\hxb^*_{i-1}, W_i)] \approx \Var[\hxb_{i-1} - \hxb^*_{i-1}]$. 
Thus, we can rewrite Equation~\ref{eqn:a-decompose} and get  
$$
\Var[\hab_i - \hab_i^*] \approx \Var[\hxb_{i-1} - \hxb^*_{i-1}] + C
$$
With Equation~\ref{eqn:recurrent}, we have
\begin{align*}
\Var[\hxb_i - \hxb^*_i] &= \beta^2_{i, i} \Var[\hab_i - \hab_i^*] + (1 - \beta^2_{i, i}) \Var[\hxb_i - \hxb^*_i] \\
& \approx \beta^2_{i, i} (\Var[\hxb_{i-1} - \hxb^*_{i-1}] + C) + (1 - \beta^2_{i, i}) \Var[\hxb_i - \hxb^*_i] \\
& = \Var[\hxb_i - \hxb^*_i] + \beta^2_{i, i} C 
\end{align*}
Therefore, we have $\Var[\cF(\xb_0, W) - \cF(\xb_0, W^*)] \approx \sum_{i=1}^N \beta^2_{i, i} C$.
\end{proof}

\section{\our Implementation Details}
\label{appendix:implement}

As introduced in Section~\ref{subsec:admin}, we introduce a new set of parameters to rescale the module outputs. 
Specifically, we refer these new parameters as $\omega$ and construct the Post-LN sub-layer as:
$$
\xb_{i} = \LN (\bbb_i), \mbox{where}\, \bbb_i=\xb_{i-1} \cdot \bomega_i + f_i (\xb_{i-1})
$$
where $\cdot$ is the element-wise product. 

After training, \our can be reparameterized as the conventional Post-LN structure (\ie, removing $\bomega_i$). 
Specifically, we consider $\xb_i = \frac{\bbb_i}{\sigma_{b}}\bgamma + \bnu$. 
Then, for feedforward sub-layers, we have
$$
\bbb_{i} = \xb_{i-1} \cdot \omega +\max(0, \xb_{i-1}\wf)\wff, \mbox{where}\, \xb_{i} = \frac{\bbb_{i-1}}{\sigma_b}\bgamma + \bnu. 
$$
It can be reparameterized by changing $\bgamma$, $\bnu$, $\wf$ to $\bgamma\bomega_i$, $\bnu\bomega_i$, $\frac{1}{\bomega_i}\wf$ respectively, \ie, 
$$
\bbb'_{i} = \xb'_{i-1} + \max(0, \xb'_{i-1}\frac{1}{\bomega_i}\wf)\wff, \mbox{where}\, \xb'_{i-1} = \frac{\bbb'_{i-1}}{\sigma_{b}}\bgamma\bomega_i + \bnu\bomega_i. 
$$
For Self-Attention sub-layers, we have 
$$
\bbb_{i} = \xb_{i-1} + \sum_{h} \softmax(\xb_{i-1} \wq_h \wk_h \xb_{i-1})\xb_{i-1} \wv_h \wo_h, \mbox{where}\, \xb_{i} = \frac{\bbb_{i-1}}{\sigma_b}\bgamma + \bnu.
$$ 
It can be reparameterized by changing $\bgamma$, $\bnu$, $\wq_h$, $\wk_h$, $\wv_h$ to $\bgamma\bomega_i$, $\bnu\bomega_i$, $\frac{1}{\bomega_i}\wq_h$, $\frac{1}{\bomega_i}\wk_h$ $\frac{1}{\bomega_i}\wv_h$ respectively, \ie, 
$$
\bbb'_{i} = \xb'_{i-1} + \sum_{h} \softmax(\xb'_{i-1} \frac{1}{\bomega_i}\wq_h \wk_h \frac{1}{\bomega_i}\xb'_{i-1})\xb'_{i-1} \frac{1}{\bomega_i}\wv_h \wo_h, \mbox{where}\, \xb'_{i-1} = \frac{\bbb'_{i-1}}{\sigma_{b}}\bgamma\bomega_i + \bnu\bomega_i.
$$
For Encoder-Attention sub-layers, we have 
$$
\bbb_{i} = \xb_{i-1} + \sum_{h} \softmax(\xb_{i-1} \wq_h \wk_h \xb^{\cdot e})\xb^{\cdot e} \wv_h \wo_h, \mbox{where}\, \xb_{i} = \frac{\bbb_{i-1}}{\sigma_b}\bgamma + \bnu.
$$ 
It can be reparameterized by changing $\bgamma$, $\bnu$, $\wq_h$ to $\bgamma\bomega_i$, $\bnu\bomega_i$, $\frac{1}{\bomega_i}\wq_h$ respectively, \ie, 
$$
\bbb'_{i} = \xb'_{i-1} + \sum_{h} \softmax(\xb'_{i-1} \frac{1}{\bomega_i}\wq_h \wk_h \xb^{\cdot e})\xb^{\cdot e} \frac{1}{\bomega_i}\wv_h \wo_h, \mbox{where}\, \xb'_{i-1} = \frac{\bbb'_{i-1}}{\sigma_{b}}\bgamma\bomega_i + \bnu\bomega_i.
$$
It is easy to find $\bbb'_i = \bbb_i$ in all three situations. 

From the previous analysis, it is easy to find that introducing the additional parameter $\bomega_i$ is equivalent to rescale some model parameters. 
In our experiments on IWSLT14 De-En, we find that directly rescaling initialization parameters can get roughly the same performance with introducing $\bomega_i$. 
However, it is not very stable when conducting training in a half-precision manner. 
Accordingly, we choose to add new parameters $\bomega_i$ instead of rescaling parameters. 

\section{Experimental Setup}
\label{appendix:exp}

Our experiments are based on the implementation from the fairseq package~\citep{ott2019fairseq}.
As to pre-processing, we follow the public released script from previous work~\citep{ott2019fairseq,lu2020understanding}. 
For WMT'14 datasets, evaluations are conducted on the provided `newstest14` file, and more details about them can be found in \citet{bojar2014findings}. 
For the IWSLT'14 De-En dataset, more analysis and details can be found in \citet{cettolo2014report}. 

\begin{table}[t]
\centering
\caption{ReZero Performance on IWSLT'14 De-En. Models are Transformer-small w. 6-layer encoder \& decoder.}
\label{tab:rezero}
\begin{tabular}{r|ccccc}
\toprule
Models & Admin & Post-LN & Pre-LN & ReZero & ReZero+Post-LN \\
\midrule
BLEU & 35.67$\pm$0.15 & 35.64$\pm$0.23 & 35.50$\pm$0.04 & 33.67$\pm$0.14 & 34.67$\pm$0.08 \\
\bottomrule
\end{tabular}
\end{table}

\begin{table}[t]
\centering
\caption{Performance and model size on WMT'14 En-Fr (AL-BL refers A-layer encoder \& B-layer decoder).}
\label{tab:wmt14fr}
\begin{tabularx}{\linewidth}{lccYY}
\toprule
Methods & Param. \# & dim($W^{(1)}$) in FFN & Enc\#-Dec\# & BLEU \\
\midrule
T5-Base~\cite{raffel2019exploring} & 220 M & $512\times2048$ & 6L-6L & 41.2 \\
T5-Large~\cite{raffel2019exploring} & 770 M & $1024\times4096$ & 12L-12L & 41.5 \\
T5-3B~\cite{raffel2019exploring} & 3 B & $1024\times16384$ & 24L-24L & 42.6 \\
T5-11B~\cite{raffel2019exploring} & 11 B & $1024\times65536$ & 24L-24L & 43.4 \\
\midrule
Trans.Big-RNMT+~\cite{Chen2018TheBO} & 377 M &  $1024\times8192$ & 6L-6L & 41.12 \\
DynamicConv~\cite{wu2018pay} & 213 M & $1024\times4096$ & 7L-7L & 43.2 \\
DG-Transformer~\cite{Wu2019DepthGF} & 264 M & $1024\times4096$ & 8L-8L & 43.27 \\
Prime~\cite{zhao2019muse} & 252 M & $1024\times4096$ & 6L-6L & 43.48 \\
\midrule
Pre-LN (60L--12L) & 262 M &  $512\times2048$ & 60L-12L & 43.10 \\
Admin (60L--12L) & 262 M & $512\times2048$ & 60L-12L & 43.80 \\
\bottomrule
\end{tabularx}
\end{table}

As to model specifics, we directly adopt Transformer-small configurations on the IWSLT'14 De-En dataset and stacks more layers over the Transformer-base model on the WMT'14 En-De and WMT'14 En-Fr datasets.
Specifically, on the IWSLT'14 De-En dataset, we use word embedding with 512 dimensions and 6-layer encoder/decoder with 4 heads and 1024 feedforward dimensions; on the WMT'14 En-De and WMT'14 En-Fr datasets, we use word embedding with 512 dimension and 8-head encoder/decoder with 2048 hidden dimensions. 
Label smoothed cross entropy is used as the objective function with an uncertainty $= 0.1$~\citep{szegedy2016rethinking}. 

For Model training, we use RAdam as the optimizer~\cite{Liu2019OnTV} and adopt almost all hyper-parameter settings from \citet{lu2020understanding}.
Specifically, for the WMT'14 En-De and WMT'14 En-Fr dataset, all dropout ratios (including (activation dropout and attention dropout) are set to 0.1. 
For the IWSLT'14 De-En dataset, after-layer dropout is set to $0.3$, and a weight decay of $0.0001$ is used. 
As to optimizer, we set $(\beta_1, \beta_2) = (0.9, 0.98)$,  use inverse sqrt learning rate scheduler with a warmup phrase (8000 steps on the WMT'14 En-De/Fr dataset, and 6000 steps on the IWSLT'14 De-En dataset). 
The maximum learning rate is set to $1e^{-3}$ on the WMT'14 En-De dataset and $7e^{-4}$ on the IWSLT'14 De-En and WMT'14 En-Fr datasets. 
We conduct training for $100$ epochs on the WMT'14 En-De dataset, $90$ epochs on the IWSLT'14 De-En dataset and $50$ epochs on the WMT'14 En-Fr dataset, while the last 10 checkpoints are averaged before inference. 

On the IWSLT'14 De-En dataset, we conduct training on one NVIDIA GeForce GTX 1080 Ti GPU and set the maximum batch size to be $4096$. 
On the WMT'14 En-De dataset, we conduct training on four NVIDIA Quadro R8000 GPUs and set maximum batch size (per GPU) as $8196$. 
On the WMT'14 En-Fr dataset, we conduct training with the Nvidia DGX-2 server (6L-6L uses 4 NVIDIA TESLA V100 GPUs and 60L-16L uses 16 NVIDIA TESLA V100 GPUs) and set the maximum batch size (per GPU) as $8000$ for 6L-6L and $5000$ for 60L-16L. 
On the IWSLT'14 De-En dataset, Transformer-small models (w. 37 M Param.) take a few hours to train.
On the WMT'14 En-De dataset, 6L-6L models (w. 63 M Param.) take $\sim1$ day to train, 
12L-12L (w. 107M Param.) models take $\sim2$ days to train, 
and 18L-18L (w. 151M Param.) models take $\sim3$ days to train. 
On the WMT'14 En-Fr dataset, 6L-6L models (w. 67 M Param.) takes $\sim2$ days to train, and 60L-12L models (w. 262M Param.) takes $\sim 2.5$ days to train. 
All training is conducted in half-precision with dynamic scaling (with a 256-update scaling window and a 0.03125 minimal scale).
All our implementations and pre-trained models would be released publicly. 

\section{Comparison to ReZero}
\label{appendix:rezero}

Here, we first conduct comparisons with ReZero~\cite{Bachlechner2020ReZeroIA} under two configurations--the first employs the original ReZero model, and the second adds layer normalizations in a Post-LN manner. 
As summarized in Table~\ref{tab:rezero}, the ReZero initialization leads to a performance drop, no matter layer normalization is used or not. 
It verifies our intuition that over small dependency restricts the model potential. 
At the same time, we find that adding layer normalization to ReZero helps to improve the performance. 
Intuitively, as dropout plays a vital role in regularizing Transformers, layer normalization helps to not only stabilize training but alleviate the impact of turning off dropouts during the inference. 

\section{Performance on the WMT'14 En-Fr}
\label{appendix:enfr}

To explore the potential of \our, we conduct experiments with 72-layer Transformers on the WMT'14 En-Fr dataset (with a 60-layer encoder and 12-layer decoder, we add less layers to decoder to encourage the model to rely more on the source context). 

As in Table~\ref{tab:wmt14fr}, \our (60L--12L) achieves a BLEU score of 43.80, the new state-of-the-art on this long-standing benchmark. 
This model has a 60-layer encoder and a 12-layer decoder, which is significantly deeper than other baselines.
Still, since the number of parameters increases in a quadratic speed with regard to hidden dimensions and a linear speed with regard to layer numbers, our model has roughly the same number of parameters with other baselines. 
It is worth mentioning that \our even achieves better performance than all variants of pre-trained T5 models, which demonstrates the great potential of our proposed method.
Also, \our achieves a better performance than Pre-LN (60L--12L), which further verifies that the Pre-LN architecture restricts deep models' potential.



\end{document}